\documentclass{aptpub}

\authornames{B. Hajek \emph{et al.}} 
\shorttitle{Recovering a Hidden Community} 

\usepackage{amsmath,amsfonts,amssymb,bm, verbatim,dsfont,mathtools}
\usepackage{color,graphicx,appendix}
\usepackage{subfigure}
\usepackage{etoolbox}

\usepackage{tikz}
\usetikzlibrary{arrows,calc}

\usepackage{pgfplots}

\pgfplotsset{
    standard/.style={
        axis x line=middle,
        axis y line=middle,
        every axis x label/.style={at={(current axis.right of origin)},anchor=west},
        every axis y label/.style={at={(current axis.above origin)},anchor=south}
    }
}

\usepackage{xr,xspace}
\usepackage{todonotes}
\usepackage{caption,soul}
\usepackage{savesym}
\savesymbol{OR}
\usepackage{algorithm}
\usepackage{algorithmic}
\makeatletter

\usepackage{xspace,prettyref}

\newcommand{\diverge}{\to\infty}

\newcommand{\reals}{{\mathbb{R}}}

\newcommand{\naturals}{{\mathbb{N}}}

\newcommand{\eexp}{e}

\newcommand{\identity}{\mathbf I}
\newcommand{\allones}{\mathbf J}

\newcommand{\diff}{{\rm d}}

\newcommand{\blue}{\color{blue}}
\newcommand{\nb}[1]{{\sf\blue[#1]}}

\newcommand{\Expect}{\mathbb{E}}
\newcommand{\expect}[1]{\mathbb{E}\left[ #1 \right]}
\newcommand{\eexpect}[1]{\mathbb{E}[ #1 ]}

\newcommand{\Prob}{\mathbb{P}}
\newcommand{\pprob}[1]{ \mathbb{P}\{ #1 \} }
\renewcommand{\prob}[1]{ \mathbb{P}\left\{ #1 \right\} }

\renewcommand{\var}{\mathsf{var}}
\newcommand{\Cov}{\text{Cov}}

\newcommand{\Bern}{{\rm Bern}}
\newcommand{\Binom}{{\rm Binom}}
\newcommand{\Pois}{{\rm Pois}}

\newcommand{\eg}{e.g.\xspace}

\newrefformat{eq}{(\ref{#1})}
\newrefformat{chap}{Chapter~\ref{#1}}
\newrefformat{sec}{Section~\ref{#1}}
\newrefformat{alg}{Algorithm~\ref{#1}}
\newrefformat{fig}{Fig.~\ref{#1}}
\newrefformat{tab}{Table~\ref{#1}}
\newrefformat{rmk}{Remark~\ref{#1}}
\newrefformat{clm}{Claim~\ref{#1}}
\newrefformat{def}{Definition~\ref{#1}}
\newrefformat{cor}{Corollary~\ref{#1}}
\newrefformat{lmm}{Lemma~\ref{#1}}
\newrefformat{prop}{Proposition~\ref{#1}}
\newrefformat{app}{Appendix~\ref{#1}}
\newrefformat{hyp}{Hypothesis~\ref{#1}}
\newrefformat{thm}{Theorem~\ref{#1}}
\newrefformat{ass}{Assumption~\ref{#1}}   

\newcommand{\pth}[1]{\left( #1 \right)}

\newcommand{\sth}[1]{\left\{ #1 \right\}}

\newcommand{\indc}[1]{{\mathbf{1}_{\left\{{#1}\right\}}}}

\newcommand{\diag}[1]{\mathsf{diag} \left\{ {#1} \right\} }

\newcommand{\calL}{{\mathcal{L}}}

\newcommand{\calN}{{\mathcal{N}}}

\newcommand{\calT}{{\mathcal{T}}}

\renewcommand{\hat}{\widehat}
\renewcommand{\tilde}{\widetilde}

\usepackage{ifthen}

\newcommand{\type}{arXiv} 

\newcommand{\nbnew}[1] %
{\ifthenelse{\equal{\type}{both}}{\nb{#1}}{}}
\newcommand{\arxivonly}[1] %
{\ifthenelse{\equal{\type}{arxiv}}{#1}{}}

 \ifthenelse{\equal{\type}{arXiv} }{\renewcommand{\crnotice}{\hfill{\today}} }{}

\begin{document}

\title{Recovering a Hidden Community Beyond \\ the Kesten-Stigum Threshold in $O(|E| \log^*|V|)$ Time } 

\authorone[University of Illinois at Urbana-Champaign]{Bruce Hajek} 
\authortwo[Yale University]{Yihong Wu}
\authorthree[Purdue University]{Jiaming Xu}
\addressone{Department of ECE and Coordinated Science Lab,
University of Illinois at Urbana-Champaign, Urbana, IL 61801} 
\addresstwo{Department of Statistics and Data Science,
Yale University, New Haven, CT 06511}
\addressthree{Krannert School of Management, Purdue University, West Lafayette, IN 47907}
\begin{abstract}
Community detection is considered for a stochastic block model graph of
$n$ vertices, with $K$ vertices in the planted community, edge probability $p$ for
pairs of vertices both in the community, and edge probability $q$ for other pairs of vertices.
 The main focus of the paper is on weak recovery of the community based on the
graph $G$, with $o(K)$ misclassified vertices on average, in the sublinear regime $n^{1-o(1)} \leq K \leq o(n).$
A critical parameter is the effective signal-to-noise ratio $\lambda=K^2(p-q)^2/((n-K)q)$, with $\lambda=1$ corresponding to the Kesten-Stigum threshold. We show that  a belief propagation algorithm achieves weak recovery 
if $\lambda>1/e$, beyond the Kesten-Stigum threshold by
a factor of $1/ \eexp.$ The belief propagation algorithm only needs to run for $\log^\ast n+O(1) $ iterations,
with the total
time complexity $O(|E| \log^*n)$, where $\log^*n$ is the iterated logarithm of $n.$
Conversely, if $\lambda  \leq 1/e$, no local algorithm
can asymptotically outperform trivial random guessing. Furthermore, a linear message-passing algorithm 
that corresponds to applying power iteration to the non-backtracking matrix of the graph 
is shown to attain weak recovery if and only if $\lambda>1$.
In addition, the belief propagation algorithm can be combined with a linear-time voting procedure to achieve the information limit of  exact recovery 
(correctly classify all vertices with high probability) for all 
$K \ge \frac{n}{\log n} \left( \rho_{\rm BP} +o(1) \right),$ where $\rho_{\rm BP}$ is a function of $p/q$.  
\end{abstract}




\renewcommand{\thefootnote}{\arabic{footnote}}
\setcounter{footnote}{0}

\section{Introduction}

\nbnew{Comments such as this one appear only if the  ``type"  variable in the latex file is set to ``both."}
The problem of finding a densely connected subgraph in a large graph arises in many
research disciplines such as theoretical computer science, statistics, and theoretical physics.
To study this problem, the stochastic block model  \cite{Holland83} for a single dense  community 
is considered.

\begin{definition}[Planted dense subgraph model] \label{def:pds_model}
Given $n\geq 1,$ $C^\ast  \subset [n]$, and  $0 \leq q \leq  p \leq 1,$  the corresponding
{\em planted dense subgraph model} is a random undirected
graph $G=(V,E)$ with $V=[n],$  such that two vertices are connected by an edge with
probability $p$ if they are both in $C^\ast$, and with probability $q$ otherwise,
with the outcomes being mutually independent for distinct pairs of vertices.
\end{definition}

The terminology is motivated by the fact that the subgraph induced by the community
$C^\ast$ is typically denser than the rest of the
graph if $p>q$~\cite{McSherry01,arias2013community, ChenXu14, HajekWuXu14,Montanari:15OneComm}.
The problem of interest is  to recover $C^\ast$ based on the graph $G$.

We consider a sequence of planted dense subgraphs indexed by $n$ and assume $p$ and $q$ depend on $n.$
For a given $n$, the set $C^*$ could be deterministic or random.
We also introduce $K\geq 1$ depending on $n$, and assume either that  $|C^*|\equiv K$ or
$|C^*| / K \to 1$ in probability as $n\to\infty.$   Where it matters we specify which assumption holds.
Since the focus of this paper is to understand the fundamental limits of recovering the hidden community in the
planted dense subgraph model, we assume the model parameters $(K, p, q)$ are known to the estimators\footnote{
It remains open whether this assumption can be relaxed without changing the fundamental limits of recovery.
The paper \cite{Decelle11} suggests a method for
estimating the parameters but it is unclear how to incorporate it into our theorems.}. For simplicity,
we further impose the mild assumptions that $K/n$ is bounded away from one and $p/q$ is bounded from above.
We primarily focus on two types of recovery guarantees.

\begin{definition}[Exact Recovery]
Given an estimator $\hat{C}=\hat{C}(G) \subset [n]$,   $\hat{C}$ {\em exactly recovers} $C^\ast $ if
$
\lim_{n\to \infty} \pprob{\hat C \neq C^\ast} =0,
$
where the probability is taken with respect to the randomness of $G$ and with respect to possible randomness
in $C^*$ and the algorithm for generating $\hat C$ from $G.$
\end{definition}

Depending on the application, it may be enough to ask for an estimator $\hat{C}$ which almost
completely agrees with $C^\ast.$

\begin{definition}[Weak Recovery] \label{def:weakrecovery}
Given an estimator $\hat C=\hat C(G)  \subset [n]$,  $\hat C$ {\em weakly recovers} $C^\ast$ if, as
$n \to \infty$, $\frac{1}{K} | \hat C  \triangle C^* |  \to 0,$  where the convergence is in probability, and
$\triangle$ denotes the set difference.
\end{definition}

Exact and weak recovery are the same as strong and weak consistency, respectively, as defined in \cite{Mossel14}.
Clearly an estimator that exactly recovers $C^*$ also weakly recovers $C^*.$
Also, it is not hard to show that the existence of an estimator satisfying  \prettyref{def:weakrecovery} is equivalent to
the existence of an estimator such that $ \eexpect{| \hat C  \triangle C^* |}  = o(K)$ (see \cite[Appendix A]{HajekWuXu_one_info_lim15} for a proof).

Intuitively, if the community size $K$ decreases, or $p$ and $q$ get closer,  recovery of the community becomes harder.
A critical role is played by the parameter
\begin{equation}
\lambda=\frac{K^2(p-q)^2}{(n-K)q},	
	\label{eq:lambda}
\end{equation}
 which can be interpreted as the effective signal-to-noise ratio for classifying a vertex according to its degree.
It turns out that if the community size scales \emph{linearly} with the network size, optimal recovery can be achieved via degree-thresholding in linear time. 
For example, if $K\asymp n-K\asymp n$ and $p/q$ is bounded, 
a na\"{i}ve degree-thresholding algorithm can attain weak recovery in time linear in the number of edges,
provided that $\lambda \to \infty$, which is information theoretically necessary when $p$ is bounded away from one. 
Moreover,  one can show that degree-thresholding followed by a linear-time voting procedure achieves exact
recovery whenever it is information theoretically possible in this asymptotic regime  (see \prettyref{app:degreethreshold} for a proof).

Since it is easy to recover a hidden community of size $K=\Theta(n)$ weakly or exactly up to the information limits, we next turn to the \emph{sublinear} regime where $K=o(n)$.
However, detecting and recovering polynomially small communities of size $K=n^{1-\Theta(1)}$ is known \cite{HajekWuXu14} to suffer a fundamental computational barrier (see \prettyref{sec:connections} for details).
In search for the critical point where statistical and computational limits depart, 
the main focus of this paper is in the slightly sublinear regime of $K=n^{1-o(1)}$ and $np  = n^{o(1)}$ and analysis of
the belief propagation (BP) algorithm for community recovery.

The belief propagation algorithm is an iterative algorithm which aggregates the likelihoods
computed in the previous iterations with the observations in the current iteration.
Running belief propagation for one iteration and then thresholding the beliefs reduces to degree thresholding.
Montanari \cite{Montanari:15OneComm} analyzed the performance of the belief propagation algorithm  for community recovery in a different regime with $p=a/n$, $q=b/n$,
 and $K=\kappa n$, where $a,b, \kappa$ are assumed to be fixed as $n \to \infty$.
 In the limit where first $n \to \infty$, and then $\kappa \to 0$ and $a, b\to \infty$,
 it is shown that using a local algorithm\footnote{Loosely speaking, an algorithm is $t$-local,
if the computations determining the status of any given vetex $u$ 
depend only on the subgraph induced
by vertices whose distance  to $u$ is  at most $t.$ See~\cite{Montanari:15OneComm}
for a formal definition. In this paper, $t$ is allowed to slowly grow with $n$
so long as $(2+np)^t=n^{o(1)}.$}, namely  belief propagation running for a constant number of iterations, $\mathbb{E}[|\hat{C} \Delta C^*|] = o(n)$;
conversely, if $\lambda <1/e$, for all local algorithms, 
$\mathbb{E}[|\hat{C} \Delta C^*|] = \Omega(n).$
 However, since we focus on $K=o(n)$ and weak recovery
 demands $\mathbb{E}[|\hat{C} \Delta C^*|] = o(K)$, the following question remains unresolved:
 \emph{Is $\lambda > 1/e$  the performance limit of belief propagation algorithms for
 weak recovery when $K=o(n)$} ?

 In this paper, we answer positively this question by analyzing belief propagation running
 for $\log^\ast n+O(1)$ iterations.
Here, $\log^* (n)$ is the iterated logarithm, defined as the number of times the
 logarithm function must be iteratively applied to $n$ to get a result less than or equal to one.
We show that if $\lambda > 1/e,$  weak recovery can be
achieved by a belief propagation algorithm running for $\log^\ast(n)+O(1)$ iterations, whereas
if $\lambda < 1/e,$  all local algorithms including belief propagation cannot asymptotically outperform trivial random guessing without the observation of the graph. 

The proof is based on analyzing the analogous belief propagation algorithm to
classify the root node of a multi-type Galton-Watson tree, which is the
limit in distribution of the neighborhood of a given vertex in the original graph $G.$
In contrast to the analysis of belief propagation in \cite{Montanari:15OneComm}, where the number of iterations is held fixed regardless of the size of 
graph $n$, our analysis on the tree and the associated coupling lemmas entail the number of iterations converging slowly to infinity as the size of the graph increases, 
 in order to guarantee adequate performance of the algorithm in the case that $K=o(n)$.
 Also, our analysis is mainly based on studying the recursions of 
 exponential moments of beliefs instead of Gaussian approximations as used in   \cite{Montanari:15OneComm}.


Furthermore, we analyze a linear message passing algorithm corresponding to applying the power method to the \emph{non-backtracking matrix} of the graph~\cite{KrzakalaMMSLC13spectral,BordenaveLelargeMassoulie:2015dq},
whose spectrum has been shown to be more informative than that of the adjacency matrix for the purpose of clustering.  It is established that this
linear message passing algorithm followed by thresholding provides weak recovery if $\lambda > 1$ 
and it does not improve upon trivial random guessing asymptotically if $\lambda < 1$.

As shown in \prettyref{rmk:kesten-stigum}, the threshold $\lambda=1$ coincides with the Kesten-Stigum threshold \cite{Kesten1966additional,Mossel04}, which originated in the study of phase transitions of limiting offspring distributions of multi-type Galton-Watson trees. Since the local neighborhood of a given vertex  under stochastic
block models is a multi-type Galton-Watson tree in the limit, the Kesten-Stigum threshold also plays a critical role in the study of community detection. It was first conjectured~\cite{Decelle11} and later rigorously proved 
that for stochastic block models with two equal-sized planted communities, recovering a community partition positively correlated with the planted one is efficiently attainable if above the Kesten-Stigum threshold~\cite{Massoulie13,Mossel13,BordenaveLelargeMassoulie:2015dq}, while it is information-theoretically impossible if below the threshold~\cite{Mossel12}. With more than three equal-sized communities, correlated recovery is shown to be informationa-theoretically possible beyond the Kesten-Stigum threshold; however, it is conjectured that 
no polynomial-time algorithm can succeed in correlated recovery beyond the Kesten-stigum threshold~\cite{banks-etal-colt,AbbeSandon16}. In contrast, we show that in the case of a single hidden community, belief propagation algorithm achieves weak recovery efficiently beyond the Kesten-Stigum threshold by a factor of $e$.  The problems mentioned above with equal-sized communities are balanced in the sense that the
expected degree of a vertex given its community label is the same for all community labels.   The single community problem we
study is unbalanced--vertex degrees reveal information on vertex community labels.   Hence, our results do not disprove that the
Kesten-Stigum threshold is the limit for computationally tractable algorithms in the balanced case.



Finally, we address exact recovery.
  As shown in \cite[Theorem 3]{HajekWuXu_one_info_lim15}, if there is an algorithm that can provide weak recovery even if the community size
is random and only approximately equal to $K,$  then it
can be combined with a linear-time voting procedure to achieve exact recovery whenever it is information-theoretically possible.
For $K=o(n)$, we show that both the belief propagation and the linear message-passing algorithms indeed 
can be upgraded to achieve exact recovery via local voting. Somewhat surprisingly, belief propagation plus voting achieves the information limit of exact recovery  if  $K \ge \frac{n}{\log n} \left( \rho_{\rm BP} (p/q) +o(1) \right),$ where $\rho_{\sf BP}(c) \triangleq \frac{1}{e (c-1)^2} ({1 - \frac{c-1}{\log c} \log \frac{e \log c}{c-1} })$. 

\section{Related work} \label{sec:connections}
The problem of recovering a single community demonstrates a fascinating interplay between statistics and computation and a potential departure between computational
and statistical limits.

In the special case of $p=1$ and $q=1/2$, the problem of finding one community reduces to the classical  planted clique problem \cite{Jer92}.
If the clique has size $ K \le 2(1-\epsilon) \log_2 n $ for any $\epsilon>0$, then it cannot be uniquely determined;
if $ K\ge 2(1+\epsilon) \log_2 n $, an exhaustive search finds the clique with high probability.
In contrast, polynomial-time algorithms are only known to find a clique of
size $K \ge c \sqrt{n} $ for any constant $c>0$~\cite{Alon98,Feige10findinghidden,Dekel10,ames2011plantedclique}, and it is shown in \cite{Deshpande12}
 that  if $K \ge (1+\epsilon) \sqrt{n/\eexp}$, the clique can be found in $O(n^2 \log n)$ time with high probability
and  $\sqrt{n/\eexp}$ may be a fundamental limit for solving the planted clique problem in nearly linear
time in the number of edges in the graph. Recent work \cite {Meka15} shows that the degree-$r$ sum-of-squares (SOS) relaxation cannot find the clique unless $K \gtrsim (\sqrt{n}/\log n)^{1/r}$;
an improved lower bound $K \gtrsim n^{1/3}/\log n$  for the degree-$4$ SOS  is proved in \cite{DeshpandeMontanari15}.
Further improved lower bounds are obtained recently in \cite{HKP15,RS15}. 

Another recent work \cite{HajekWuXu14} focuses on the case $p=n^{-\alpha}$, $q=cn^{-\alpha}$ for fixed constants $c<1$ and $0<\alpha<1$, and
$K=\Theta(n^{\beta})$ for $0<\beta<1$.
It is shown that no polynomial-time algorithm can attain the information-theoretic threshold of detecting the planted dense subgraph unless the planted clique problem can be solved in polynomial time (see \cite[Hypothesis 1]{HajekWuXu14} for the precise statement). For exact recovery, MLE succeeds with high probability if $\alpha<\beta<\frac{1}{2}+\frac{\alpha}{4}$; however, no randomized polynomial-time solver exists, conditioned on the same planted clique hardness hypothesis.

In sharp contrast to the computational barriers discussed in the previous two paragraphs, in the regime
 $p=a \log n /n$ and $q=b\log n/n$ for fixed $a,b$ and $K=\rho n$ for a fixed constant $0<\rho<1$,
recent work \cite{HajekWuXuSDP14} derived a function $\rho^*(a,b)$ such that
if $\rho >\rho^*$, exact recovery is achievable in polynomial-time via
semidefinite programming relaxations of ML estimation; if $\rho <\rho^*$, any estimator fails to exactly recover
the cluster with probability tending to one regardless of the computational costs.

In summary, the previous work revealed that for exact recovery, 
 a significant gap between the information limit and the limit of polynomial-time algorithms emerges as the community size $K$ decreases from $K=\Theta(n)$ to $K=n^{\beta}$ for $0<\beta<1$. In search of the exact phase
transition point where information and computational limits depart, the present paper further zooms into the regime
of $K = n^{1-o(1)}$. We show in \prettyref{app:comparsion_info} that  belief propagation plus voting attains the sharp information limit
 if $K \ge \frac{n}{\log n} (\rho_{\sf BP}(p/q) + o(1))$.  However, as soon as $ \lim_{n\to \infty} K \log n/ n \le \rho_{\sf BP}(p/q) $,
 we observe a gap between the information limit and the necessary condition of local algorithms, given by $\lambda>1/\eexp$. See \prettyref{fig:exact_phase_diagram} for an illustration. 
 For weak recovery,  as soon as $K=o(n)$,  a gap between the information limit and the necessary condition of local algorithms  emerges. 

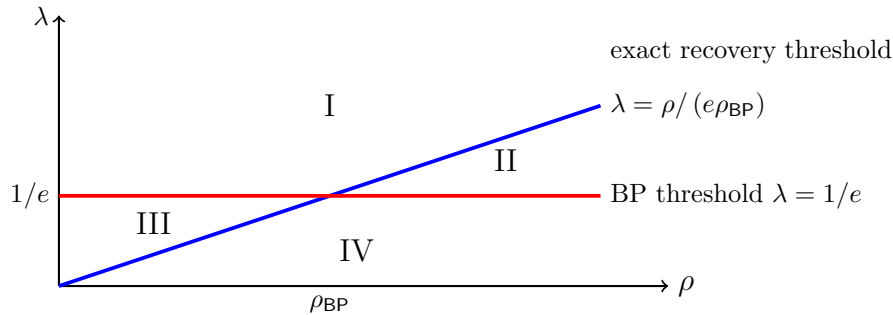
\begin{figure}[h!]
     \begin{center}
\scalebox{1}{
\begin{tikzpicture}[xscale = 1.8,yscale = 2.4, thick]
\node [below] at (2, 0) {$\rho_{\sf BP}$};
\draw[->] (0, 0) -- (4.5, 0) node [below, right]{\large $\rho$};
\draw[->] (0, 0) -- (0, 1.5) node [left]{$\lambda$};
\draw[line width=0.5mm,blue] (0,0)--(4,1);
\draw[line width=0.5mm,red] (0,0.5) -- (4,0.5);
\node [left] at (0, 0.5){$1/e$};
\node [right] at (4,1){$\lambda=  \rho/ \left(e \rho_{\sf BP} \right) $ };
\node [right] at (4,1.3){exact recovery threshold};
\node [right] at (4, 0.5){BP threshold $\lambda=1/e$};
 \node at (2,1){\large I};
 \node at (3.3,0.7){\large II};
 \node at (2.2,0.2){\large IV};
 \node at (0.7,0.35){\large III};
\end{tikzpicture}}
\end{center}
\caption{Phase diagram with $K= \rho n / \log n$ and $p/q=c$ for fixed constants $c\ge 1$, $\rho$, and
$\lambda$ as $n \to \infty$.
 In region I, exact recovery is provided by the BP algorithm plus voting procedure.
 In region II, weak recovery is provided by the BP
 algorithm, but exact recovery is not information theoretically possible.
 In region III exact recovery is information theoretically possible, but no polynomial-time algorithm is known for even weak recovery.   In region IV, with $\lambda >0$ and $\rho>0$, weak recovery,
 but not exact recovery, is information theoretically possible and no polynomial time algorithm is known for weak recovery.
}
\label{fig:exact_phase_diagram}
\end{figure}


\section{Main results 
} \label{sec:main}
As mentioned above, in search
for the critical point where statistical and computational
limits depart, we focus on the regime where $K$ is slightly
sublinear in $n$  and invoke the following assumption.

\begin{assumption}\label{ass:mainassumption}  
As $n\to\infty,$   $p\geq q$,  $p/q=O(1)$, $n^{1-o(1)} \leq K \leq o(n)$,
and $\lambda$ is a positive constant.
\end{assumption}

\subsection{Upper and lower bounds for belief propagation}
Let $\sigma \in \{0,1\}^n$ denote the indicator vector of $C^\ast$
and $A$ denote the adjacency matrix of the graph $G.$
To detect whether a given vertex $i$ is in the community, a natural approach is to compare the log likelihood 
ratio $\log  \frac{\prob{G| \sigma_i =1 } }{ \prob{G | \sigma_i =0 } } $ to a certain threshold.  However, it is often 
computationally expensive to evaluate the log likelihood ratio.  As we show in this paper, when the average degree
scales as $n^{o(1)}$, the neighborhood of vertex $i$ is tree-like with high probability as long as the radius $t$ of
the neighorhood satisfies $(2+np)^t=n^{o(1)}$; moreover, 
on the tree, the log likelihoods can be  exactly computed  in a finite recursion via belief propagation. 
These two observations together suggest the following belief propagation algorithm for approximately computing
the log likelihoods for the community recovery problem (See \prettyref{lmm:BPTree} for derivation of belief propagation
algorithm on tree).
Let $\partial i$ denote the set of neighbors of $i$ in $G$ and
$$
\nu \triangleq \log  \frac{n-K}{K},
$$
which is equal to the log prior ratio $\log \frac{\prob{\sigma_i=0}}{\prob{\sigma_i=1}}$.
Define the message transmitted from vertex $i$ to its neighbor $j$ at $(t+1)$-th iteration
as
\begin{align}  \label{eq:mp_commun}
R^{t+1}_{i \to j} = - K(p-q) + \sum_{ \ell \in \partial i \backslash \{ j\} }   \log \left(  \frac{ \eexp^{ R^{t}_{\ell \to i}  - \nu } \left( \frac{p}{q} \right) + 1  }{ \eexp^{R^{t}_{\ell \to i}  -\nu } +1 } \right)
\end{align}
for initial conditions  $R^{0}_{i \to j} = 0$ for all $i \in [n]$ and $j \in \partial i$.
Then we approximate $\log \frac{\prob{ G| \sigma_i =1 } }{ \prob{G | \sigma_i =0 } }$
by the belief of vertex $i$ at $(t+1)$-th iteration, $R^{t+1}_i$, which is determined by combining incoming messages
from its neighbors as follows:
\begin{align}   \label{eq:mp_combine_commun}
R^{t+1}_{i} = - K(p-q) + \sum_{ \ell \in \partial i } \log \left(  \frac{ \eexp^{ R^{t}_{\ell \to i}   -\nu }\left( \frac{p}{q} \right)   + 1  }{ \eexp^{R^{t}_{\ell \to i} - \nu } +1 } \right).
\end{align}

\begin{algorithm}[htb]
\caption{Belief propagation for weak recovery}\label{alg:MP_commun}
\begin{algorithmic}[1]
\STATE Input: $n,  K \in \naturals.$ $p>q >0$, adjacency matrix $A \in \{0,1\}^{n\times n}$, $t_f \in \naturals$
\STATE Initialize: Set  $R^{0}_{i \to j}=0$ for all $i  \in [n]$ and $j \in \partial i$.
\STATE Run $t_f-1$ iterations of belief propagation as in \prettyref{eq:mp_commun} to compute $R^{t_f-1}_{i\to j}$ for all $i \in [n]$ and $j \in \partial i$.
\STATE Compute  $R_{i}^{t_f }$ for all $i \in [n]$ as per \prettyref{eq:mp_combine_commun}.
\STATE Return $\hat{C},$ the set of  $K$ indices in $[n]$ with largest values of $R_{i}^{t_f}.$
\end{algorithmic}
\end{algorithm}

%
%

\begin{theorem}  \label{thm:BP_Bernoulli}    
Suppose \prettyref{ass:mainassumption} holds with $\lambda  > 1/\eexp$ and $(np)^{\log^* \nu}  = n^{o(1)}.$ 
Let $t_f = \bar t_0+\log^*(\nu ) + 2,$   where $\bar{t}_0$ is a constant depending
only on $\lambda.$
Let $\hat{C}$ be produced by
Algorithm \ref{alg:MP_commun}.   
If
the planted dense subgraph model (\prettyref{def:pds_model}) 
is such that $|C^*|\equiv K,$ then for any constant $r>0$, there exists $\nu_0(r)$ 
such that for all $\nu \ge \nu_0(r)$, 
\begin{equation}   \label{eq:BP_weak_fixed_size}
  \eexpect{  |C^* \triangle \hat{C}|  }  \leq   n^{o(1)}   + 2 K\eexp^{-\nu r}.     
\end{equation}
If instead $|C^*|$ is random with   $\prob{  \big| |C^*| - K\big| \geq \sqrt{3K\log n}  }  \leq n^{-1/2+o(1)}$,
then 
\begin{equation}    \label{eq:BP_weak_random_size}
  \eexpect{  |C^* \triangle \hat{C}|  } \leq  n^{\frac{1}{2}+ o(1)}  +  2K \eexp^{-\nu r}.  
\end{equation}
For either assumption about $|C^*|$, weak recovery is achieved:  $\expect{  |C^* \triangle \hat{C}|} =o(K).$
The running time is $O(|E(G)| \log^* n)$,  where $|E(G)|$ is the number of edges in the graph $G.$
\end{theorem}
We remark that the same conclusion also holds
for the estimator $\hat C_o = \{ i : R_{i}^{t_f} \geq \nu \}$, 
but returning a constant size estimator $\hat{C}$ leads to simpler analysis of the algorithm for exact recovery.



Next we discuss how to use the belief propagation (BP) algorithm to  achieve exact recovery. The key idea is
to attain exact recovery in two steps. In the first step, we apply BP
for weak recovery. In the second step, we use a linear-time local voting procedure to clean-up the residual errors made by BP.  In particular, for each vertex $i$, we count $r_i$, the number of neighbors in the community estimated by BP,
and pick the set of $K$ vertices with the largest values of $r_i$.
To facilitate analysis, we adopt the {\em successive withholding} method described in
\cite{Mossel14,HajekWuXu_one_info_lim15} to ensure the first and second step are independent of each other.
In particular, we first randomly partition
the set of vertices into a finite number of subsets.   One at a time,
one subset is withheld to produce a reduced set of vertices, to which
BP is applied.   The
estimate obtained from the reduced set of vertices is used to classify the
vertices in the withheld subset.  The idea is to gain independence:
the outcome of BP based on the reduced set of vertices is
independent of the data corresponding to edges between the
withheld vertices and the reduced set of vertices. The full description of the
algorithm is given in \prettyref{alg:MPplus_Bernoulli}.

\begin{algorithm}
\caption{Belief propagation  plus cleanup for exact recovery}\label{alg:MPplus_Bernoulli}
\begin{algorithmic}[1]
\STATE Input: $n \in \naturals$, $K >0$, $p>q >0$, adjacency matrix $A \in \{0,1\}^{n\times n}$, $t_f \in \naturals,$
and $\delta \in (0,1)$ with $1/\delta, n\delta \in \naturals.$
\STATE (Partition): Partition $[n]$ into $ 1/\delta$ subsets $S_k$ of size $n\delta,$  uniformly at random.
\STATE (Approximate Recovery) For each $k=1, \ldots,  1/\delta $, let $A_k$ denote the restriction of $A$ to the rows and columns with index
in $[n]\backslash S_k$, run  Algorithm \ref{alg:MP_commun} (belief propagation for weak recovery)
with input $(n(1-\delta), \lceil K(1-\delta)\rceil , p, q , A_k, t_f)$
and let $\hat{C}_k$ denote the output.
\STATE (Cleanup) For each $k=1, \ldots,  1/\delta $  compute $r_i=\sum_{j \in \hat{C}_k } A_{ij}$ for all $i \in S_k$ and return
$\tilde{C}$, the set of $K$ indices in $[n]$ with the largest values of $r_i.$
\end{algorithmic}
\end{algorithm}


\begin{theorem}  \label{thm:MP_plus_Bernoulli}   
Suppose \prettyref{ass:mainassumption} holds with $\lambda  > 1/\eexp$ and  $(np)^{\log^* \nu}  = n^{o(1)}.$
Consider the planted dense subgraph model (\prettyref{def:pds_model}) with $|C^*|\equiv K.$
Select $\delta > 0$ so small that $(1-\delta)\lambda \eexp > 1$. 
Let $t_f = \bar t_0+\log^*(\nu ) + 2,$   where $\bar{t}_0$ is a constant depending
only on $\lambda(1-\delta).$ 
Also, suppose $p$ is bounded away from $1$ and 
the following condition
is satisfied:
\begin{align}
\liminf_{n \to \infty}  \frac{ K  d(\tau^\ast \| q) }{\log n } > 1,
\label{eq:planted_dense_exact_suff1XX}
\end{align}
where
\begin{align}
\tau^\ast = \frac{ \log \frac{1-q}{1-p} + \frac{1}{K} \log \frac{n}{K} }{\log \frac{p(1-q)}{q(1-p) } } \label{eq:deftau}
\end{align}
and 
$d(p\|q) = 
p \log \frac{p}{q} + (1-p)\log \frac{1-p}{1-q}$ denotes the Kullback-Leibler divergence between Bernoulli
distributions with mean $p$ and $q$.
Let $\tilde{C}$ be produced by Algorithm \ref{alg:MPplus_Bernoulli}.
Then $\Prob\{\tilde{C} = C^*\} \to 1$  as $n \to \infty$.
The running time is $O(|E(G)| \log^* n).$
\end{theorem}

Note that the condition \prettyref{eq:planted_dense_exact_suff1XX} is shown
in \cite{HajekWuXu_one_info_lim15} to be the necessary (if ``$>$" is replaced by ``$\geq$") and sufficient condition for the success of clean-up procedure
in upgrading  weak recovery to exact recovery.

 \ifthenelse{\equal{\type}{APT}}{}{
 We comment briefly on some implementation issues for \prettyref{alg:MPplus_Bernoulli}.
The assumption $n\delta \in \naturals$ is an integer is only for notational
convenience.    If we drop that assumption, and continue to assume
 $\frac 1 \delta \in \naturals,$  and if  $n\geq \left(  \frac 1 \delta + 1\right)^2$,
we could partition $[n]$ into $\frac 1 \delta + 1$ subsets, the first
$\frac 1 \delta$ of which have cardinality $\lfloor n \delta \rfloor,$ and the
last of which has cardinality less than or equal to $\lfloor n \delta \rfloor.$ 
The proof of \prettyref{thm:MP_plus_Bernoulli} then goes through
with minor modifications.   Also, the constant $\delta$ does not need to
be extremely small to allow $\lambda$ to be reasonably close to $1/\eexp.$
For example,  if we take $\delta=1/11,$ the condition on $\lambda$ in
\prettyref{thm:MP_plus_Bernoulli}  becomes  $\lambda > \frac{1.1}{\eexp}.$
}

Next, we provide a lower bound  on the error probability achievable by any local algorithm for estimating
the label $\sigma_u$ of a given vertex $u$. 
Let   $p_e= \pi_0p_{e,0} + \pi_1p_{e,1}$ for prior probabilities $\pi_0=(n-K)/n$  and $\pi_1=K/n,$
where $p_{e,0}=\prob{ \hat{\sigma}_u =1 | \sigma_u=0}$ and $p_{e,1}= \prob{ \hat{\sigma}_u =0 | \sigma_u=1}.$

\begin{theorem}
[Converse for local algorithms]
 \label{thm:planted_BP_converse}
Suppose \prettyref{ass:mainassumption} holds with  $0 < \lambda  \leq 1/ \eexp.$
Let $t_f \in \naturals$ depend on $n$ such that  $(2+np)^{t_f }=n^{o(1)}.$
Consider the planted dense subgraph model (\prettyref{def:pds_model}) with $C^*$ random and uniformly
distributed over all subsets of $[n]$ such that $|C^*|\equiv K.$
Then for any estimator $\hat C$ such that  for each vertex $u$ in $G$,  $\sigma_u$ is estimated based
on $G$ in a neighborhood of radius $t_f$ from $u,$
\begin{align} \label{eq:lower_bnd_comm_Bhat}
\Expect[| \hat C \triangle C^*| ] \geq  \frac{K(n-K)}{n}\exp(-\lambda \eexp / 4) - n^{o(1)}.
\end{align}
and
\begin{align}  \label{eq:Psucc_bnd}
 p_{e,0}+p_{e,1} \geq \frac{1}{2}e^{-1/4} - n^{-1+o(1)}.
\end{align}
Furthermore, $ \liminf_{n\to\infty}  \frac{np_e}{K} \geq 1,$ or,
equivalently,
\begin{align}   \label{eq:BP_converse_I}
\liminf_{n\to\infty} \frac{\Expect[| \hat C \triangle C^*| ] }{K} \geq 1.
\end{align}
\end{theorem}

The assumption $(2+np)^{t_f }=n^{o(1)}$ is needed to ensure the neighborhood of
radius $t_f$ from any given vertex $u$ is a tree with high probability. 

Note that an estimator is said to achieve weak recovery in \cite{Montanari:15OneComm},  if
$\lim_{n \to \infty} p_{e,0}+p_{e,1}=0$. Condition \eqref{eq:Psucc_bnd} shows that weak recovery in this sense
is not possible. If $C^*$ is uniformly distributed over $\{ C \subset [n] : |C|=K\},$
among all estimators that disregard the graph, the
one that minimizes the mean number of classification errors is $\hat C \equiv \emptyset$ (declaring no community), which achieves   $\frac{\eexpect{ |\hat C \triangle C^*| }}{K} = 1$,
or equivalently,
$p_e= K/n$.   Condition \eqref{eq:BP_converse_I} shows that in the asymptotic regime
$\nu \to \infty$ with $\lambda < 1/\eexp$, improving upon random guessing is impossible.

\subsection{Upper and lower bounds for linear message passing}

Results are given in this section  to show that a particular spectral method -- linear message passing --
achieves weak recovery if and only if $\lambda>1.$
Spectral algorithms estimate the communities based on the
principal eigenvectors of the adjacency matrix, see, \eg, 
\cite{Alon98,McSherry01,YunProutiere14} and the reference therein.
Under the single community model,
$\expect{A}= (p-q) ( \sigma  \sigma^\top  -\diag{\sigma} ) + q ( \allones- \identity),$
where $\diag{\sigma}$ denotes the diagonal matrix
with the diagonal entries given by $\sigma$; $\identity $ denotes the identity matrix
and $\allones$ denotes the all-one matrix.
By the Davis-Kahan $\sin\theta$ theorem \cite{DavisKahan70},
the principal eigenvector of $A- q ( \allones- \identity)$ is almost parallel to $ \sigma $  provided that
the spectral norm $\|A-\expect{A}\|$ is much smaller than $K(p-q)$; thus one
can estimate $C^\ast$ by thresholding the principal eigenvector entry-wise.
Therefore, if we apply the spectral method, a natural matrix to start with is $A-q(\mathbf{J}-\mathbf{I})$,  or
$A - q\mathbf{J}.$    Finding the principal eigenvector of $A - q\mathbf{J}$ according to the power method
is done by starting with some vector and repeatedly multiplying by $A - q\mathbf{J}$ sufficiently many times.
We shall consider the scaled matrix $\frac{ A - q\mathbf{J}  }{\sqrt{m}}$ where $m=(n-K)q.$
Of course the scaling doesn't change the eigenvectors.
This suggests the following linear message passing update equation:
\begin{align}
\theta_{i}^{t+1}=   -\frac{q}{\sqrt{m}}     \sum_{\ell  \in [n] }   \theta^t_{ \ell }  +    \frac{1}{\sqrt{m}}      \sum_{\ell \in \partial i } \theta^t_{ \ell }.
\label{eq:spectralpowermethod}
\end{align}
The first sum is over all vertices in the graph and doesn't depend on $i.$    An idea is to appeal to the
law of large numbers and replace the first sum by its expectation.    Also,
in the sparse graph regime $np=o(\log n)$, there exist vertices of high degrees $\omega(np)$, and the spectrum of $A$ is very sensitive
to high-degree vertices (see, \eg, \cite[Appendix A]{HajekWuXuSDP14} for a proof). To deal with this issue, as proposed in \cite{KrzakalaMMSLC13spectral,BordenaveLelargeMassoulie:2015dq}, we associate the messages in \prettyref{eq:spectralpowermethod}
with directed edges and prevent the message transmitted from $j$ to $i$ from being immediately
reflected back as a term in the next message from $i$ to $j,$ resulting in the following linear message passing algorithm:
\begin{align}
\theta_{i\to j}^{t+1} &=- \frac{q ((n-K)A_t  + K B_t) }{\sqrt{m}}  +     \frac{1}{\sqrt{m} } \sum_{\ell \in \partial i\backslash \{j\} } \theta^t_{ \ell \to i }
 \label{eq:spectral_mp1}.
\end{align}
with initial values $\theta^0_{ \ell \to i }=1,$
where  $A_t \approx \Expect[ \theta_{\ell \to i}^{t} | \sigma_\ell=0]$ and $B_t \approx  \Expect[ \theta_{\ell \to i}^{t} | \sigma_\ell=1].$
Notice that when computing $\theta_{i\to j}^{t+1}$, the contribution of $\theta_{j \to i}^t$ is subtracted out.
Since we focus on the regime $np=n^{o(1)}$, the graph is locally tree-like with high probability.
In the Poisson random tree limit of the neighborhood of a vertex, the expectations $\Expect[ \theta_{\ell \to i}^{t} | \sigma_\ell=0]$
and  $ \Expect[ \theta_{\ell \to i}^{t} | \sigma_\ell=1]$ can be calculated
exactly, and as a result we take  $A_0=1,$  $A_t=0$ for $t\geq 1,$  and $B_t = \lambda^{t/2}$ for $t\geq 0.$

The update equation  \prettyref{eq:spectral_mp1} can be expressed in terms of the {\em non-backtracking matrix}
associated with graph $G$.    
It is the matrix $\mathbf{B} \in \{0,1\}^{2m \times 2m}$ with
$B_{e f} = \indc{ e_2=f_1} \indc{e_1 \neq f_2}$, where $e=(e_1,e_2)$ and $f=(f_1,f_2)$
are directed edges. Let ${\Theta}^t \in \reals^{2m}$ denote the
messages on directed edges with ${\Theta}_e^t =\theta^t_{e_1 \to e_2}$.
Then, \prettyref{eq:spectral_mp1} in matrix form reads
\begin{align*}
{\Theta}^{t+1} =  -\frac{q ((n-K)A_t  + K B_t) }{\sqrt{m}}  \mathbf{1} + \frac{1}{\sqrt{m} } \mathbf{B}^\top {\Theta}^t .
\end{align*}
As shown in \cite{BordenaveLelargeMassoulie:2015dq},   the spectral properties of the non-backtracking matrix closely
match those of the original adjacency matrix.  It is therefore reasonable to take the linear
update equation  \prettyref{eq:spectral_mp1} as a form of spectral method for
the community recovery problem. Finally, to estimate $C^\ast$, we
define the belief at vertex $u$ as:
\begin{align}
\theta_{u}^{t+1} & =   - \frac{q ((n-K)A_t  + K B_t) }{\sqrt{m}}  +     \frac{1}{\sqrt{m} } \sum_{i \in \partial u } \theta^t_{ i \to u },
 \label{eq:spectral_mp2}
\end{align}
and select the vertices $u$ such that $\theta_u^{t}$ exceeds a certain threshold. The full description of the algorithm
 is  given in  \prettyref{alg:SP_commun}.

\begin{algorithm}[htb]
\caption{Spectral algorithm for weak recovery}\label{alg:SP_commun}
\begin{algorithmic}[1]
\STATE Input: $n,  K \in \naturals.$ $p>q >0$, adjacency matrix $A \in \{0,1\}^{n\times n}$
\STATE  Set $\lambda=\frac{K^2(p-q)^2}{(n-K)q}$ and $T= \lceil 2 \alpha \frac{\log \frac{n-K}{K}}{\log \lambda}\rceil$, where $\alpha=1/4$ (in fact any $\alpha < 1$ works).
\STATE Initialize: Set  $\theta^{0}_{i \to j}=1$ for all $i  \in [n]$ and $j \in \partial i$.
\STATE Run $T-1$ iterations of message passing as in \prettyref{eq:spectral_mp1} to compute $\theta^{T-1}_{i\to j}$ for all $i \in [n]$ and $j \in \partial i$.
\STATE Run one more iteration of message passing to compute $\theta_{i}^{T}$ for all $i \in [n]$ as per \prettyref{eq:spectral_mp2}.
\STATE Return $\hat{C},$ the set of  $K$ indices in $[n]$ with largest values of $\theta_{i}^T.$
\end{algorithmic}
\end{algorithm}

\begin{theorem}  \label{thm:SP_bernoulli_spectral}    
Suppose  \prettyref{ass:mainassumption} holds with  $\lambda > 1$ and
$(np)^{\log (n/K)} =n^{o(1)}$. 
Consider the planted dense subgraph model (\prettyref{def:pds_model}) with
$$
\prob{ \big|~ |C^*| - K\big| \geq \sqrt{3K\log n}  }  \leq n^{-1/2+o(1)}.
$$
 Let $\hat{C}$ be the estimator produced by
Algorithm \ref{alg:SP_commun}.   Then  $  \expect{  |C^* \triangle \hat{C}|  } =o(K) $.
\end{theorem}

One can upgrade the weak recovery result of linear message passing to exact recovery under condition $\lambda >1$
and condition \prettyref{eq:planted_dense_exact_suff1XX}, in a similar manner as described in \prettyref{alg:MPplus_Bernoulli} and
the proof of \prettyref{thm:MP_plus_Bernoulli}.

%
%
%
The next converse shows that if  $\lambda \leq 1$ then estimating
better than the random guessing by linear message
passing is not possible.
\begin{theorem}
[Converse for linear message passing algorithm]  
\label{thm:spectral_weak_converse_A}
Suppose  \prettyref{ass:mainassumption} holds with  $0 < \lambda\leq 1$ and   
consider the planted dense subgraph model (\prettyref{def:pds_model}) with $C^*$ random and uniformly
distributed over all subsets of $[n]$ such that $|C^*|\equiv K.$
Assume $t\in \naturals,$  with $t$ possibly depending on $n$ such that
$(np)^t = n^{o(1)}$  and $t=O(\log\left(  \frac{n-K}{K}\right) ).$
Let $(\theta^t_u: u \in [n])$ be computed using the message passing updates
\prettyref{eq:spectral_mp1} and \prettyref{eq:spectral_mp2} and let
$\hat C =\{ u : \theta_u^t \geq \gamma \}$ for some threshold $\gamma,$ which may also depend on $n.$
Equivalently,  $\sigma_u$ is estimated for each $u$ by $\hat \sigma_u = \indc{\theta_u^t\geq \gamma}.$
Then  $\liminf_{n\to\infty} \frac{p_e n}{K} \geq 1.$
\end{theorem}

The proofs of \prettyref{thm:SP_bernoulli_spectral}  and \prettyref{thm:spectral_weak_converse_A}
are similar to the counterparts for belief propagation and
 \ifthenelse{\equal{\type}{arXiv}}{are given in \prettyref{app:spectral_limit}.}
{can be found in the arXiv version of this paper  \cite{HajekWuXu_one_beyond_spectral15}.}
\nbnew{This refers to appendix for proofs if type is arXiv, otherwise it refers to the arXiv version (1 of 3 places).}

 \section{Inference problem on a random tree by belief propagation}  \label{sec:BP_Bernoulli_tree}

In the regime we consider, the graph is locally tree like, with mean degree converging to infinity.
We begin by deriving the exact belief propagation algorithm for
an  infinite tree network, and then deduce performance results for using that same algorithm
on the original graph.   


The related inference problem on a Galton-Watson tree with Poisson numbers of offspring
is defined as follows.  Fix a vertex $u$  and let $T_u$ denote the
 infinite Galton-Watson undirected tree rooted at vertex $u$.    The neighbors of vertex $u$ are
 considered to be the children of vertex $u$, and $u$ is the parent of those children. The other neighbors
 of each child are the children of the child, and so on.
 For vertex $i$ in $T_u$,
let $T_i^t$ denote the subtree of $T_u$ of height $t$ rooted at vertex $i,$ induced by the set of vertices
consisting of vertex $i$ and its descendants for $t$ generations.
Let $\tau_i \in \{0, 1\}$ denote the  label of vertex $i$ in $T_u$.
Assume $\tau_u \sim \Bern( K/n)$.
For any vertex $i \in T_u,$   let $L_i$ denote the number of its children $j$ with $\tau_j=1$, and $M_i$ denote the number of its children $j$ with $\tau_j=0$.
Suppose that $L_i \sim  \Pois( Kp )$ if $\tau_i=1,$ $L_i \sim  \Pois( Kq )$ if $\tau_i=0,$  and  $M_i \sim \Pois( (n-K)q )$ for either value of $\tau_i.$

We are interested in estimating the label of root $u$ given observation of the tree $T_u^t$. Notice that the labels
of vertices in $T_u^t$ are not observed. The probability of error for an estimator $\hat \tau_u (T_u^t)$ is defined by
\begin{align}
p_e^t\triangleq \frac{K}{n}  P( \hat \tau_u = 0  | \tau_u=1 )    + \frac{n-K}{n}P(\hat \tau_u=1 | \tau_u = 0). \label{eq:deferrortree}
\end{align}
The estimator that minimizes $p_e^t$ is the maximum a posteriori probability (MAP) estimator, which can be expressed
either in terms of the log belief ratio or log likelihood ratio:
\begin{align}  \label{eq:map_rule}
\hat{\tau}_{\rm MAP} = \indc{  \xi_u^{t} \geq 0 } =    \indc{\Lambda_u^t \geq  \nu },
\end{align}
where
\begin{align*}
 \xi_u^{t} \triangleq \log \frac{\prob{ \tau_u =1 | T_u^t}  }{ \prob{ \tau_u =0 | T_u^t} }, \quad \Lambda_u^{t} \triangleq  \log \frac{\prob{T_u^t  | \tau_u=1} }{ \prob{  T_u^t | \tau_u=0} },
\end{align*}
and $\nu= \log \frac{n-K}{K}$.
By Bayes' formula, $\xi_u^{t}= \Lambda_u^{t} - \nu$, and by definition,
$\Lambda_u^{0} = 0$.
By a standard result in the theory of binary hypothesis testing
 (due to \cite{KobayashiThomas67}, stated without proof in \cite{Poor94Book},
 proved in special case $\pi_0=\pi_1=0.5$ in \cite{Kailath67}, and same proof
 easily extends to general case),
 the probability of error for the MAP decision rule is bounded by
\begin{equation}  \label{eq:Bhatt}
\pi_1\pi_0   \rho_B^2 \leq p_e^t \leq    \sqrt{\pi_1\pi_0}  \rho_{B},
\end{equation}
where the Bhattacharyya coefficient (or Hellinger integral) $\rho_{B}$ is defined by
$\rho_{B} =  \eexpect{\eexp^{\Lambda_u^t/2}| \tau_u=0 },$  and $\pi_1$ and $\pi_0$ are
the prior probabilities on the hypotheses.

We comment briefly on the parameters of the model.
The distribution of the tree $T_u$ is determined by the three parameters
$\lambda = \frac{K^2(p-q)^2}{(n-K)q}$,  $\nu$, and the ratio, $p/q.$
Indeed, vertex $u$ has label $\tau_u=1$ with probability $\frac{K}{n}=\frac{1}{1+\eexp^{\nu}},$
and the  mean numbers  of children of a vertex $i$ are given by:
\begin{eqnarray}
\expect{L_i | \tau_i=1} & = & Kp = \frac{\lambda(p/q)\eexp^{\nu}}{(p/q-1)^2}  \label{eq:mean_deg11}   \\
\expect{L_i|\tau_i=0} & = & Kq = \frac{\lambda\eexp^{\nu}}{(p/q-1)^2}   \label{eq:mean_deg01} \\
\expect{M_i} & = & (n-K)q = \frac{\lambda\eexp^{2\nu}}{(p/q-1)^2}.   \label{eq:mean_deg0}
\end{eqnarray}
The parameter $\lambda$ can be interpreted as a signal to noise ratio in case $K \ll n$ and $p/q = O(1),$  because  $\var{M_i} \gg \var{L_i}$
and $$
\lambda = \frac{  \left(  \expect{ M_i+L_i | \tau_i=1 } - \expect{ M_i+L_i | \tau_i=0 } \right)^2  }{\var{M_i}}.
$$
In this section, the parameters are allowed to vary with $n$ as long as  $\lambda > 0$ and   $p/q > 1,$  although
the focus is on the asymptotic regime: $\lambda$ fixed, $p/q = O(1),$
and $\nu \to  \infty.$    This entails that the mean numbers of children
given in \prettyref{eq:mean_deg11}-\prettyref{eq:mean_deg0} converge to infinity.
Montanari \cite{Montanari:15OneComm} considers the case of $\nu$ fixed with
$p/q\to 1,$ which also leads to the mean vertex degrees converging to infinity.
\begin{remark}\label{rmk:kesten-stigum}
It turns out that $\lambda=1$ coincides with the Kesten-Stigum threshold~\cite{Kesten1966additional}. To see this,
let $O=(O_{ab})$ denote the $2 \times 2 $ matrix with $O_{ab}$
equal to the expected number of childen of type $b$ given a parent of type $a$
for $a, b \in \{0,1\}$. 
Then 
$$
O=\begin{bmatrix}
(n-K) q  & K q\\
(n-K) q   & K p
\end{bmatrix}
.
$$
Let $\lambda_+ \ge \lambda_-$ denote the two largest eigenvalues of $M$.
The Kesten-Stigum threshold~\cite{Kesten1966additional} is defined to be $\lambda_-^2/\lambda_+ =1.$
Direct calculation gives 
$$
\lambda_{\pm}= \frac{1}{2} \left(  nq+ K(p-q) \pm | nq - K(p-q) |\sqrt{  1+ \frac{4K^2(p-q) q }{ \left( nq - K(p-q) \right)^2} }   \right).
$$
Since $K(p-q)=o(nq)$ and $K=o(n)$, it follows that $\lambda_+ = (1+o(1)) nq$ and 
$\lambda_-=(1+o(1)) K(p-q).$ Hence, 
$$
\lambda=\left(1+o(1) \right) \frac{\lambda_-^2}{\lambda_+}. 
$$
Thus $\lambda=1$ is asymptotically equivalent to Kesten-Stigum threshold $\lambda_-^2/\lambda_+=1.$ 
\end{remark}
\medskip

It is well-known that the likelihoods can be computed via a belief propagation algorithm.
Let $\partial i$ denote the set of children of vertex $i$ in $T_u$ and $\pi(i)$ denote the
parent of $i$. For every vertex $i \in T_u$ other than $u$, define
\begin{align*}
 \Lambda^{t }_{i \to \pi(i) } \triangleq \log \frac{\prob{T_i^t  | \tau_i=1} }{ \prob{T_i^t  | \tau_i=0} }.
\end{align*}
The following lemma gives a recursive formula to compute $\Lambda_u^{t};$
no approximations are needed.
 \begin{lemma}\label{lmm:BPTree}
For $t \ge 0$,
 \begin{align*}
\Lambda^{t+1}_{u } & =  -K(p-q)  + \sum_{ \ell \in \partial u }  \log \left(  \frac{ \eexp^{ \Lambda^{t}_{\ell \to u}  - \nu } (p/q)   + 1  }{ \eexp^{\Lambda^{t}_{\ell \to u} -\nu } +1} \right), \\
 \Lambda^{t+1}_{i \to \pi(i) } & = -K(p-q)   + \sum_{ \ell \in \partial i }  \log \left(  \frac{ \eexp^{ \Lambda^{t}_{\ell \to i} - \nu } (p/q) + 1  }{\eexp^{\Lambda^{t}_{\ell \to i} - \nu}  +1 }  \right) ,  \quad   \forall  i \neq u \\
\Lambda^0_{i \to \pi(i) } & = 0 , \quad \forall i \neq u .
 \end{align*}
 \end{lemma}
\begin{proof}
The last equation follows by definition. We prove the first equation; the second one follows similarly.   A key point is
to use the independent splitting property of the Poisson distribution to give an equivalent description of the numbers of children
with each label for any vertex in the tree.     Instead of separately generating the number of children of with each label,  we can first
generate the total number of children and then independently and randomly label each child. Specifically,
for every vertex $i$ in $T_u$, let $N_i$  denote the total number of its children.  Let $d_1=Kp+(n-K)q$ and $d_2=Kq+(n-K)q=nq.$
If $\tau_i=1$ then  $N_i \sim \Pois(d_1)$, and for each child $j\in \partial i,$ independently of everything else,
$\tau_j=1$ with probability $Kp/d_1$  and $\tau_j=0$ with probability $(n-K)q/d_1.$
If   $\tau_i=0$ then  $N_i \sim \Pois( d_0)$, and for each child $j\in \partial i,$ independently of everything else,
$\tau_j=1$ with probability $K/n$ and $\tau_j=0$ with probability $(n-K)/n.$    With this view, the observation
of the total number of children $N_u$ of vertex $u$ gives some information on the label of $u$,
and then the conditionally independent messages from those children give additional information.   To be precise, we have that
\begin{align*}
 \Lambda_u^{t+1}  & =  \log \frac{\prob{T_u^{t+1}  | \tau_u=1 } }{ \prob{  T_u^{t+1}  | \tau_u=0}  } \overset{(a)}{=} \log \frac{\prob{N_u | \tau_u=1 }  }{ \prob{N_u | \tau_u= 0}  } + \sum_{i \in \partial u} \log  \frac{ \prob{ T^t_i | \tau_u=1 } } {  \prob{ T_i^t  | \tau_u=0 } }   \\
 & \overset{(b)}{=} -K(p-q) + N_u  \log \frac{ d_1}{d_0}   +   \sum_{ i \in \partial u}  \log \frac{ \sum_{x\in \{0, 1\} } \prob{\tau_i=x | \tau_u=1} \prob{ T_i^t | \tau_i =x} }{
 \sum_{\tau_i \in \{0, 1\} } \prob{\tau_i=x | \tau_u=0} \prob{ T_i^t | \tau_i =x} }  \\
 & \overset{(c)}{=} -K(p-q)  + \sum_{i \in \partial u} \log \frac{ Kp  \prob{ T_i^t | \tau_i =1} + (n-K) q  \prob{ T_i^t | \tau_i =0} }{
 K q \prob{ T_i^t | \tau_i =1} +(n-K) q \prob{ T_i^t | \tau_i =0}   }\\
 & \overset{(d)}{=} -K(p-q)  + \sum_{i \in \partial u} \log \frac{ \eexp^{ \Lambda^{t}_{i \to u} -\nu} (p/q)  + 1 }{ \eexp^{\Lambda^{t}_{i \to u} - \nu} +1   },
\end{align*}
where $(a)$ holds because $N_u$ and $T^t_i$ for $i \in \partial u$ are independent conditional on $\tau_u$;
$(b)$ follows because $N_u \sim \Pois(d_1)$ if $\tau_u=1$ and $N_u \sim \Pois(d_0)$ if $\tau_u=0$, and $T_i^t$ is independent of $\tau_u$ conditional on $\tau_i$;
$(c)$ follows from the fact $\tau_i\sim \Bern(Kp/d_1)$ given $\tau_u=1$, and $\tau_i\sim \Bern(Kq/d_0)$ given $\tau_u=0;$
$(d)$ follows from the definition of $\Lambda^{t}_{i \to u}$.
\end{proof}

Notice that $\Lambda_u^t$ is a function of $T_u^t$ alone; and it is statistically
correlated with the vertex labels.   Also,  since the construction of a subtree $T_i^t$ and
its vertex labels is the same as the construction of $T_u^t$ and its vertex labels,
the conditional distribution of $T_i^t$ given $\tau_i$ is the same as the conditional
distribution of $T_u^t$ given $\tau_u.$   Therefore, for any $i \in \partial u,$ the conditional
distribution  of   $\Lambda_{i\to u}^t $ given $\tau_i$ is the same as the conditional
distribution of $\Lambda_u^t$ given $\tau_u.$
For $i=0$ or $1$, let $Z^t_i$ denote a random variable that has the same distribution as $\Lambda_u^t$ given $\tau_u=i$.
The above update rules can be viewed as an infinite-dimensional recursion that 
determines the probability distribution of $Z_0^{t+1}$ in terms of that of $Z_0^{t}$.   

The remainder of this section is devoted to the analysis of belief propagation on the Poisson tree model, and is organized into two main parts.    
In the first part, \prettyref{sec:expo_moments} gives expressions
for exponential moments of the log likelihood messages, which are applied in \prettyref{sec:bd_from_moments}
to yield an upper bound, in \prettyref{lmm:bp_tree_error} on the error probability for the problem of classifying the root vertex of the tree.   That
bound, together with a standard coupling result between
Poisson tree and local neighborhood of $G$ (stated in~\prettyref{app:coupling_lemmas}), 
is enough to establish weak recovery
for the belief propagation algorithm run on  graph $G,$ given in \prettyref{thm:BP_Bernoulli}.
The second part of this section focuses on lower bounds on the probability of correct classification in \prettyref{sec:lower_bnd_tree}. 
Those bounds,  together with the coupling lemmas, are used  to establish the converse results for local algorithms.

\subsection{Exponential moments of log likelihood messages for Poisson tree}   \label{sec:expo_moments}

The following lemma gives formulas for some exponential moments
of $Z^t_0$ and $Z^t_1$, based on \prettyref{lmm:BPTree}.    Although the formulas
are not recursions, they are close enough to permit useful analysis.
%

\begin{lemma}\label{lmm:expZ}
For $t \ge 0$ and any integer $h\geq 2,$
\begin{align}  
\expect{\eexp^{hZ^{t+1}_0 } }& = \expect{\eexp^{(h-1)Z^{t+1}_1 } }  \nonumber  \\
&=\exp\left\{  K(p-q)\sum_{j=2}^h \binom{h}{j} \left( \frac{\lambda}{K(p-q)} \right)^{j-1}  \expect{ \left(  \frac{ \eexp^{  Z_1^t } } { 1+ \eexp^{ Z_1^t - \nu }  }\right)^{j-1} }   \right\}.
 \label{eq:exphZ1}
\end{align}
\end{lemma}
%
\begin{proof}
We first illustrate the proof for $h=2.$ 
By the definition of $\Lambda_u^t$ and change of measure, we have
$
 \expect{ g (\Lambda_u^t) | \tau_u=0} = \eexpect{ g( \Lambda_u^t ) \eexp^{- \Lambda_u^t} | \tau_u=1},
$
 where $g$ is any measurable function such that the expectations above are well-defined. It follows that
 \begin{align}   \label{eq:change_measure}
 \expect{ g(Z_0^t) } = \eexpect{ g(Z_1^t) \eexp^{-Z_1^t} }.
 \end{align}
 Plugging $g(z) = \eexp^{z}$ and $g(z)=\eexp^{2z}$, we have that $\expect{\eexp^{Z_0^t} } = 1$ and  
  $\expect{ \eexp^{ 2Z_0^t} } = \expect{ \eexp^{Z_1^t} }.$
 Moreover,
 \begin{align}
  \eexp^{\nu}  \expect{ g(Z_0^t ) } + \expect{g(Z_1^t ) } = \expect{ g(Z_1^t ) ( \eexp^{-Z_1^t +\nu } + 1 ) }. \label{eq:changemeasure}
 \end{align}
 Plugging $g(z) = (1+\eexp^{-z+\nu} )^{-1}$ and $g(z) =  ( 1+\eexp^{-z+\nu} )^{-2} $
  into the last displayed equation,
 we have
 \begin{align}
\eexp^{\nu}  \expect{ \frac{1}{1+\eexp^{-Z_0^t+\nu}  } } +\expect{ \frac{1}{ 1+\eexp^{-Z_1^t+\nu}  } }  &= 1, \label{eq:symmetry1}\\
 \eexp^{\nu}  \expect{ \frac{1}{ ( 1+\eexp^{-Z_0^t+\nu} )^2 } } +\expect{ \frac{1}{ (1+\eexp^{-Z_1^t+\nu} )^2 } } & =  \expect{ \frac{1}{ 1+\eexp^{-Z_1^t+\nu}  } }
 \label{eq:symmetry2}.
 \end{align}
In view of \prettyref{lmm:BPTree}, by defining $f(x)=\frac{ x (p/q) +1 }{x +1}$, we get that
\begin{align*}
\eexp^{2 \Lambda_u^{t+1} }   = \eexp^{-2 K(p-q) }
\prod_{ \ell \in \partial u }  f^2 \left( \eexp^{ \Lambda^{t}_{\ell \to u} - \nu } \right).
\end{align*}
Since the distribution of $\Lambda^t_{\ell \to u}$ conditional on $\tau_u=0$ and $\tau_u=1$ is the same as the distribution of $Z_0^t$ and $Z_1^t$, respectively,  it follows that
\begin{align*}
\expect{\eexp^{2 Z^{t+1}_0 } }= \eexp^{-2 K(p-q) }
  \expect{ \left( \expect{ f^2 \left(  \eexp^{ Z^{t+1}_1 - \nu } \right) } \right)^{L_u}}
 \expect{ \left( \expect{ f^2 \left(  \eexp^{ Z^{t+1}_0 - \nu } \right) } \right)^{M_u}}.
\end{align*}
Using the fact that $\expect{c^X} = \eexp^{\lambda (c-1) }$ for $X \sim \Pois(\lambda)$ and $c>0$, we have
\begin{align*}
& \expect{\eexp^{2 Z^{t+1}_0 } } = e^{ -2  K(p-q) + Kq\left(  \expect{  f^2 \left( \eexp^{ Z^{t+1}_1 - \nu }  \right) } -1 \right)+ (n-K) q\left( \expect{ f^2 \left(     \eexp^{ Z^{t+1}_0 - \nu } \right)} -1 \right) }.
\end{align*}
Notice that
\begin{align*}
f^2(x) = \left(  1+  \frac{ p/q  -1  }{1+ x^{-1}} \right)^2 =1+  \frac{2 (p/q-1)}{ 1+ x^{-1} } +    \frac{ (p/q-1)^2}{ \left( 1+ x^{-1} \right)^2  } .
\end{align*}
It follows that
\begin{align*}
&Kq\left(  \expect{  f^2 \left( \eexp^{ Z^{t+1}_1 - \nu }  \right) } -1 \right)
+ (n-K) q\left( \expect{  f^2 \left(  \eexp^{ Z^{t+1}_0 - \nu  } \right) } -1 \right)   \\
&=  2 K q (p/q-1) \left(  \expect{ \frac{ 1 } {1+ \eexp^{-  Z_1^t +\nu } } }    + \eexp^{\nu} \expect{ \frac{1 } {1+ \eexp^{-  Z_0^t +\nu } } }   \right) \\
& + K q (p/q-1)^2 \left(  \expect{  \frac{ 1 } { ( 1+ \eexp^{-  Z_1^t +\nu } )^2  }   } + \eexp^{\nu}  \expect{  \frac{ 1 } { ( 1+ \eexp^{-  Z_0^t +\nu } )^2  }   }  \right) \\
& \overset{(a)}{=} 2K (p-q) + K q (p/q-1)^2  \expect{  \frac{ 1 } { 1+ \eexp^{-  Z_1^t +\nu }  }  } \\
& = 2K (p-q) + \lambda \expect{  \frac{ \eexp^{  Z_1^t } } { 1+ \eexp^{ Z_1^t - \nu }  }  },
\end{align*}
where $(a)$ follows by applying \prettyref{eq:symmetry1} and \prettyref{eq:symmetry2}.
Combining the above proves  \prettyref{eq:exphZ1} with $h=2$.   
For general $h \ge 2$, we expand $f^h(x) = \left(  1+  \frac{ p/q  -1  }{1+ x^{-1}} \right)^h$ using binomial
coefficients as already illustrated for $h=2.$   
\end{proof}

Using the notation
\begin{align}
a_t = & ~ \expect{  \eexp^{  Z_1^t } }   \label{eq:defn_at}    \\
b_t = & ~ \expect{  \frac{ \eexp^{  Z_1^t } } { 1+ \eexp^{ Z_1^t - \nu }  }  },    \label{eq:defn_bt}
\end{align}
\prettyref{eq:exphZ1} with $h=2$ becomes
\begin{equation}
	a_{t+1} = \exp(\lambda b_t).
	\label{eq:atbt}
\end{equation}

The following lemma provides upper bounds on some exponential moments in terms of $b_t.$
\begin{lemma}  \label{lmm:exp_BP_bounds}
Let $C \triangleq \lambda(2 + \frac{p}{q})$  and  $C' \triangleq \lambda(3 + 2\frac{p}{q} +  (  \frac{p}{q} )^2  ).$
Then   $\eexpect{\eexp^{2 Z^{t+1}_1 } } \le \exp (C b_t)$ and
$\eexpect{\eexp^{3 Z^{t+1}_1 } } \le \exp (C' b_t).$
More generally, for any integer $h\geq 2,$
\begin{align}  \label{eq:gen_Cb_bnd}
&\expect{\eexp^{h Z^{t+1}_0 } }  =   \expect{\eexp^{(h-1) Z^{t+1}_1 } }  \leq
\eexp^{   \lambda b_t \sum_{j=2}^h \binom{h}{j} \left( \frac{p}{q} - 1\right)^{j-2}             }.
\end{align}
\end{lemma}
\begin{proof}
Note that   $ \frac{ \eexp^{  z} } { 1+ \eexp^{ z - \nu }  }    \leq \eexp^{\nu}$ for all $z.$   Therefore, for any $j\geq 2,$
 $\left(  \frac{ \eexp^{  z} } { 1+ \eexp^{ z - \nu }  }  \right)^{j-1}  \leq \eexp^{(j-2) \nu}\left( \frac{ \eexp^{  z} } { 1+ \eexp^{ z- \nu }  } \right).$
Applying this inequality to  \prettyref{eq:exphZ1} yields \prettyref{eq:gen_Cb_bnd}.
\end{proof}

\subsection{Upper bound on classification error via exponential moments}  \label{sec:bd_from_moments}

Note that $b_t\approx a_t$ if $\nu \gg 0,$ in which case \prettyref{eq:atbt} is approximately
a recursion for $\{b_t\}$.  The following two lemmas use this intuition to show that if
$\lambda >1/\eexp$ and $\nu$ is large enough, the $b_t$'s eventually grow large.
In turn, that fact will be used to show that the Bhattacharyya coefficient mentioned in \prettyref{eq:Bhatt}, which can be expressed as $\rho_B=\eexpect{\eexp^{Z_0^t/2}} = \eexpect{\eexp^{-Z_1^t/2}}$,  becomes small,
culminating in \prettyref{lmm:bp_tree_error}, giving an upper bound on the classification error
for the root vertex.

\begin{lemma}  \label{lmm:btat}
Let $C \triangleq \lambda(2 + \frac{p}{q})$.   Then
\begin{equation} \label{eq:b_tplus1}
b_{t+1} \geq \exp(\lambda b_t)  \left(1- \eexp^{-\nu/2}  \right)~~~~\mbox{if}~~b_t\leq  \frac{\nu}{2(C-\lambda)}.
\end{equation}
\end{lemma}

\begin{proof}
Note that $C-\lambda >0.$   If $b_t\leq  \frac{\nu}{2(C-\lambda)},$ we have
\begin{align*}
b_{t+1}& \overset{(a)}{\geq}  a_{t+1} - \expect{\eexp^{-\nu+2Z_1^{t+1} } } \overset{(b)}{\geq} \eexp^{\lambda b_t}-  \eexp^{-\nu  +C b_t} \\
	&= \eexp^{\lambda b_t} \pth{1 -   \eexp^ {-\nu + (C-\lambda) b_t}}  \overset{(c)}{\geq}   \eexp^{\lambda b_t} \pth{1 - \eexp^{-\nu/2}}.
\end{align*}	
where $(a)$ follows by the definitions \prettyref{eq:defn_at} and \prettyref{eq:defn_bt} and the fact $\frac{1}{1+x} \geq 1 - x$ for $x\geq 0;$
$(b)$ follows from \prettyref{lmm:exp_BP_bounds};  (c) follows from the condition $b_t\leq  \frac{\nu}{2(C-\lambda)}.$
\end{proof}

\begin{lemma}  \label{lmm:info_monotone}
The variables $a_t$ and  $b_t$ are nondecreasing in $t$ and  $\eexpect{\eexp^{Z^t_0/2}}$ is
non-increasing in $t$ over all $t\geq 0.$   More generally,  $\expect{\Upsilon\left(\eexp^{Z^t_0}\right)}$ is nondecreasing (non-increasing)
in $t$ for any convex (concave, respectively) function $\Upsilon$ with domain $(0,\infty).$
\end{lemma}
\begin{proof}
Note that, in view of \prettyref{eq:change_measure},   $\expect{\Upsilon\left(\eexp^{Z^t_0}\right)}$ becomes
$a_t$ for the convex function $\Upsilon(x)=x^2,$  $b_t$ for the convex function $\Upsilon(x)=x^2/(1+x \eexp^{-\nu}),$
and  $\eexpect{\eexp^{Z^t_0/2}}$ for the concave function $\Upsilon(x)= \sqrt{x}.$  It thus suffices to prove the last statement
of the lemma.

It is well known that for a nonsingular binary hypothesis testing problem with a growing amount of information
indexed by some parameter $s$ (i.e. an increasing family of $\sigma$-algebras as usual in martingale theory),
 the likelihood ratio $\frac{\diff \mathbb{P}}{ \diff \mathbb{Q}}$
is a martingale under measure $\mathbb{Q}$.  Therefore, the likelihood ratios $\{ \eexp^{\Lambda_u^t}: t\geq 0 \}$
(where $\Lambda_s$ denotes the log likelihood ratio)  at the root vertex $u$
for the infinite tree, conditioned on $\tau_u=0,$  form a martingale.   Thus, the random
variables $\{\eexp^{Z_0^t}: t\geq 0\}$ can be constructed on a single probability space to be
a martingale.  The lemma therefore follows from Jensen's inequality.
\end{proof}

Recall that  $\log^* (\nu)$ denotes the number of times the logarithm function must
be iteratively applied to $\nu$ to get a result less than or equal to one.

\begin{lemma} \label{lmm:b_lower_bnd}
Suppose $\lambda > 1/\eexp.$
There are  constants $\bar t_0$ and $\nu_o >  0$ depending only on $\lambda$
such that
$$
b_{\bar{t}_0+\log^*(\nu) +2 } \geq   \exp(\lambda\nu/(2(C-\lambda))\left(1-\eexp^{-\nu/2}  \right),
$$
where $C=\lambda\left( \frac{p}{q} + 2\right),$  whenever $\nu \geq \nu_o$ and $\nu \geq 2(C-\lambda).$
\end{lemma}

\begin{proof}
 Given $\lambda $ with $\lambda >1/\eexp,$  select the following constants, depending only on $\lambda$:
\begin{itemize}
\item $D$ and $\nu_0$ so large that $\lambda \eexp^{\lambda D}\left( 1-   \eexp^{-\nu_o/2}  \right)   > 1$ and
 $\lambda \eexp  \left(1-\eexp^{-\nu_o/2} \right) \geq \sqrt{ \lambda \eexp }.$
\item  $w_0>0$ so large that $w_0\lambda\eexp^{\lambda D} \left( 1-\eexp^{-\nu_o/2}  \right) -\lambda D  \geq w_0.$
\item  A positive integer $\bar{t}_0$ so large that $\lambda ((\lambda \eexp)^{\bar{t}_0/2 -1} -D ) \geq w_0.$
\end{itemize}

Throughout the remainder of the proof we assume without further comment that $\nu \geq \nu_o$ and $\nu \geq 2(C-\lambda).$
 The latter condition and the fact $b_0=\frac{1}{1+\eexp^{-\nu}}$ ensures that $b_0 <   \frac{\nu}{2(C-\lambda)}.$
Let $t^* = \max\left\{t\geq 0: b_t <  \frac{\nu}{2(C-\lambda)}\right\}$   and let $\bar{t}_1=\log^*(\nu) .$
The first step of the proof is to show  $t^* \leq \bar t_0  + \bar t_1.$   For that purpose we will show that
the $b_t$'s increase at least geometrically to reach a certain large constant (specifically, so \prettyref{eq:bw_init} below holds), and then
they increase as fast as  a sequence produced by iterated exponentiation.

Since $b_0 \geq 0$ it follows from \prettyref{eq:b_tplus1} and the choice of $\nu_0$ that $b_1 \geq   \left(1-\eexp^{-\nu_o/2} \right) \geq (\lambda \eexp)^{-1/2}.$
Note that $\eexp^u \geq \eexp u$ for all $u>0$,  because $\frac{\eexp^u}{u}$ is minimized at $u=1.$
Thus $\eexp^{\lambda b_t} \geq \lambda \eexp b_t$, which combined with
the choice of $\nu_0$ and \prettyref{eq:b_tplus1} shows that if
$b_t\leq  \frac{\nu}{2(C-\lambda)}$ then
$b_{t+1} \geq \sqrt{\lambda \eexp }  b_t.$
It follows that $b_t \geq (\lambda \eexp)^{t/2 -1}$ for $1\leq t \leq t^*+1.$

If $b_{\bar{t}_0-1} \geq  \frac{\nu}{2(C-\lambda)}$ then $t^* \leq  \bar{t}_0-2$ and the claim $t^* \leq \bar{t}_0+\bar{t}_1$ is proved (that is, the
geometric growth phase alone was enough), so to cover
the other possibility, suppose $b_{\bar{t}_0-1} <  \frac{\nu}{2(C-\lambda)}.$
Then $\bar{t}_0  \leq  t^*+1$ and therefore  $b_{\bar{t}_0} \geq (\lambda \eexp)^{\bar{t}_0/2-1}.$
Let $t_0 = \min\{ t : b_t \geq  (\lambda \eexp)^{\bar{t}_0/2-1}  \}.$
It follows that $t_0 \leq \bar{t}_0,$  and, by the choice of $\bar t_0$ and the definition of $t_0$,
\begin{equation}  \label{eq:bw_init}
\lambda (b_{t_0} - D ) \geq w_0.
\end{equation}

Define the sequence $(w_t: t\geq 0)$ beginning with $w_0$ already chosen,
and satisfying the recursion $w_{t+1}=\eexp^{w_t}.$    It follows by induction that
\begin{equation}  \label{eq:bw_induct}
\lambda (b_{t_0+t} - D ) \geq w_t   \mbox{          for } t\geq 0, ~t_0+ t \leq t^*+1.
\end{equation}
Indeed, the base case is \prettyref{eq:bw_init}, and if   \prettyref{eq:bw_induct} holds for some $t$ with $t_0+t\leq t^*,$  then
$  b_{t_0+t}  \geq \frac{w_t}{\lambda} + D,$   so that
\begin{eqnarray*}
\lambda (b_{t_0+t+1} -D) &  \geq  & \lambda    \left(    \eexp^{\lambda b_{t_0+t} } \left(1 -\eexp^{-\nu/2} \right) - D  \right)   \\
& \geq &  w_{t+1} \lambda \eexp^{\lambda D}  (1-\eexp^{-\nu/2} )   - \lambda D  
 \geq  w_{t+1},
\end{eqnarray*}
where the last inequality follows from the choice of $w_0$ and the fact $w_{t+1} \geq w_0.$
The proof of \prettyref{eq:bw_induct} by induction is complete.

Let $\bar{t}_1 =\log^*(\nu).$    Since $w_1 \geq 1$ it follows that $w_{\bar{t}_1+1} \geq  \nu$ (verify by applying the
$\log$ function $\bar{t}_1$ times to each side).
Therefore,
$w_{\bar{t}_1+1}  \geq \frac{\lambda \nu}{2(C-\lambda)}  - \lambda D,$  where we use the
fact $C-\lambda \geq 2\lambda.$
If $t_0+\bar{t}_1 < t^*$ it would follow from \prettyref{eq:bw_induct} with $t=t_0 + \bar{t}_1+ 1$ that
$$
b_{t_0+\bar{t}_1+1}  \geq \frac{w_{\bar{t}+1}}{\lambda} + D \geq \frac{\nu}{2(C-\lambda)},
$$
which would imply $t^* \leq t_0 + \bar{t}_1,$  which would be a contradiction.   Therefore,  $t^* \leq t_0 + \bar{t}_1 \leq \bar t_0 + \bar t_1,$
as was to be shown.

Since $t^*$ is the last iteration index $t$ such that $b_t < \frac{\nu}{2(C-\lambda)},$   either $b_{t^*+1}=  \frac{\nu}{2(C-\lambda)}$, and
we say the threshold $\frac{\nu}{2(C-\lambda)}$ is exactly reached at iteration $t^*+1,$  or $b_{t^*+1} >  \frac{\nu}{2(C-\lambda)},$
in which case we say there was overshoot at iteration $t^*+1.$    First, consider the case the threshold is exactly reached at iteration $t^*+1.$
Then,  $b_{t^*+1}=  \frac{\nu}{2(C-\lambda)},$  and \prettyref{eq:b_tplus1} can be applied with $t=t^*+1,$  yielding
$$
b_{t^*+2} \geq \exp(\lambda b_{t^*+1})(1-\eexp^{-\nu/2} ) = \exp(\lambda\nu/(2(C-\lambda))(1-\eexp^{-\nu/2}  ).
$$
Since $t^*+2 \leq \bar t_0 + \bar t_1 + 2 = \bar t_0 + \log^*(\nu) + 2,$ it follows from \prettyref{lmm:info_monotone}
that $b_{\bar t_0+\log^* (\nu) + 2 } \geq b_{t^*+2},$ which
completes the proof of the lemma in case the threshold is exactly reached at iteration $t^*+1.$

To complete the proof, we explain how the information available for estimation can be reduced through a {\em thinning} method,
leading to a reduction in the value of $b_{t^*+1},$ so that we can assume without loss of generality that the threshold is
always exactly reached at iteration $t^*+1.$    Let $\phi$ be a parameter with $0 \leq \phi \leq 1.$
As before, we will be considering a total of $t^*+2$ iterations, so consider a random tree
with labels,  $(\calT^{t^*+2}_u,\tau_{\calT^{t^*+2}_u}),$  with root vertex $u$ and maximum depth $t^*+2.$
 For the original model, each vertex of depth $t^*+1$ or less with label 0 or 1 has Poisson numbers of children with labels
 0 and 1 respectively, with means specified in the construction.    For the thinning method, for each $\ell \in \partial u$ and
 each child $i$ of $\partial \ell,$  (i.e. for each grandchild of $u$)
 we generate a random variable $U_{\ell,i}$ that is uniformly distributed on the interval $[0,1].$
 Then we retain $i$ if $U_{\ell,i} \leq \phi,$ and we delete $i$, and
 all its decedents, if $U_{\ell,i} > \phi.$   That is, the grandchildren of the root vertex $u$ are each deleted with probability $1-\phi.$
 It is equivalent to reducing $p$ and $q$ to $\phi p$ and $\phi q$, respectively,  for that one generation.    Consider the calculation of the likelihood
 ratio at the root vertex for the thinned tree.   The log likelihood ratio messages begin at the leaf vertices at depth $t^*+2.$

For any vertex $\ell \neq u,$  let $\Lambda_{\ell \to \pi(\ell),\phi}$ denote
the log likelihood message passed from vertex $\ell$ to its parent, $\pi(\ell).$
Also, let $\Lambda_{u,\phi}$ denote the log likelihood computed at the
root vertex.    For brevity we leave off the superscript $t$ on the log likelihood ratios,
though $t$ on the message $\Lambda_{\ell \to \pi(\ell),\phi}$ would be $t^*+2$
minus the depth of $\ell.$  The messages of the form
$\Lambda_{\ell \to \pi(\ell),\phi}$ don't actually depend on $\phi$ unless $\ell \in \partial u.$
For a vertex $\ell \in \partial u,$  the message $\Lambda_{\ell\to u,\phi}$ has the nearly the same
representation as in \prettyref{lmm:BPTree},  namely:
 \begin{align}   \label{eq:phi_rep}
\Lambda_{\ell\to u,\phi} & =  -\phi K(p-q)   + \sum_{ i  \in \partial \ell: U_{\ell,i}\leq \phi}  \log \left(  \frac{ \eexp^{ \Lambda_{i\to \ell,\phi}  - \nu } (p/q)   + 1  }{ \eexp^{\Lambda_{i\to \ell,\phi} -\nu } +1} \right).
 \end{align}
 The representation of $\Lambda_{u,\phi}$ is the same as the representation of $\Lambda_u^{t+1}$ in \prettyref{lmm:BPTree}, except with
 $\Lambda^t_{\ell\to u}$ replaced both places on the right hand side by  $\Lambda_{\ell\to u,\phi}.$

Let $Z_{0,\phi}^t$ and $Z_{1,\phi}^t$ denote random variables
for analyzing the message passing algorithm for this depth $t^*+2$
tree. Their laws are the following.
For  $0 \leq t \leq t^*+1,$  $\calL (Z_{0,\phi}^t)$ is the law of
$\Lambda_{\ell \to \pi(\ell),\phi}$ given $\tau_{\ell}=0,$  for a vertex
$\ell$ of depth $t^*+2-t.$   And $\calL (Z_{0,\phi}^{t^*+2})$ is the
law of $\Lambda_{u,\phi}$ given $\tau_u=0.$
Note that $Z_{0,\phi}^0  \equiv 0.$
The laws $\calL (Z^t_{1,\phi})$ are determined similarly,  conditioning
on the labels of the vertices to be one.     For $t$ fixed, $\calL (Z^t_{0,\phi})$ and
$\calL (Z^t_{1,\phi})$ each determine the other because they represent
distributions of the log likelihood for a binary hypothesis testing problem.

The message passing equations for the log likelihood ratios translate
into recursions for the laws  $\calL (Z^t_{0,\phi})$  and
$\calL (Z^t_{1,\phi}).$  We have not focused directly
on the full recursions of the laws, but rather looked at equations
for exponential moments. The basic recursions we've been considering for
 $\calL (Z^t_{0,\phi})$ are exactly as before for
 $0 \leq t \leq t^*-1$ and for $t=t^*+1.$   For
 $t=t^*$ the thinning needs to be taken into
 account, resulting, for example, in the following updates for $t=t^*:$
\begin{align*}
\expect{\eexp^{Z^{t^*+1}_1 } }  =   \expect{\eexp^{2 Z^{t^*+1}_0 } }
=\exp\left\{      \lambda \phi \expect{  \frac{ \eexp^{  Z_1^{t^*} } } { 1+ \eexp^{ Z_1^{t^*} - \nu }  }  }       \right\} 
\end{align*}
and
\begin{align*}
 \expect{\eexp^{2Z^{t^*+1}_1 } }  =\exp\left\{     3 \lambda \phi \expect{  \frac{ \eexp^{  Z_1^{t^*} } } { 1+ \eexp^{ Z_1^{t^*} - \nu }  }  }
+ \frac{\lambda^2\phi}{K(p-q)}     \expect{ \left(  \frac{ \eexp^{  Z_1^{t^*} } } { 1+ \eexp^{ Z_1^{t^*} - \nu }  }\right)^2  }
   \right\}  
\end{align*}
Let
\begin{align*}
a_{t,\phi} =  ~ \expect{  \eexp^{  Z_{1,\phi}^t } }, \quad
b_{t,\phi} =  ~ \expect{  \frac{ \eexp^{  Z_{1,\phi}^t } } { 1+ \eexp^{ Z_{1,\phi}^t - \nu }  }  }
\end{align*}
for $0\leq t \leq t^*+2.$  Note that $a_{t,\phi}$ and $b_{t,\phi}$ don't depend on $\phi$
for $0\leq t \leq t^*.$   We have
\begin{equation} \label{eq:aphi}
a_{t+1,\phi}=\left\{ \begin{array}{cl}
 \exp(\lambda b_{t,\phi}) & t \neq t^*  \\
 \exp(\lambda \phi b_{t,\phi}) &  t = t^*
\end{array} \right.,
\end{equation}
We won't be needing \prettyref{eq:aphi} for $t=t^*$ but we will use it for $t=t^*+1.$

On one hand, if $\phi=0$ then $\Lambda_{\ell\to u,\phi}\equiv 0$ for all $\ell \in \partial u,$  so that
$Z_{0,\phi=0}^{t^*+1}=Z_{1,\phi=0}^{t^*+1}\equiv 0$ so that $b_{t^*+1,\phi=0}=\frac{1}{1+\eexp^{-\nu}}=\frac{n-K}{n} \le 1 <  \frac{\nu}{2(C-\lambda)}.$
On the other hand, by the definition of $t^*$ we know that  $b_{t^*+1,\phi=1} \geq \frac{\nu}{2(C-\lambda)}.$
We shall show that there  exists a value of $\phi\in [0,1]$ so that $b_{t^*+1,\phi} =  \frac{\nu}{2(C-\lambda)}.$
To do so we next prove that $b_{t^*+1,\phi}$ is a continuous, and, in fact, nondecreasing, function of $\phi,$  using
a variation of the proof of  \prettyref{lmm:info_monotone}.     Let $\ell$ denote a fixed neighbor of the root node $u.$
Note that $\eexp^{\Lambda_{\ell\to u,  \phi}  }$ is the likelihood ratio for detection of $\tau_{\ell}$ based on the
thinned subtree of depth $t^*+1$ with root $\ell.$   As $\phi$ increases from $0$ to $1$ the amount of thinning decreases,
so larger values of $\phi$ correspond to larger amounts of information.   Therefore,
conditioned on $\tau_u =0,$  $\left(\eexp^{\Lambda_{\ell,\phi}}: 0\leq \phi \leq 1\right)$ is a martingale.
Moreover, the independent splitting property of Poisson random variables imply that, given
$\tau_\ell = 0,$  the random process $\phi \mapsto |\{i\in \partial \ell: U_{\ell,i}\leq \phi \} |$ is
a Poisson process with intensity $nq,$  and therefore the sum in \prettyref{eq:phi_rep}, as a function
of $\phi$ over the interval $[0,1],$  is a compound Poisson process.  Compound Poisson processes,
just like Poisson processes, are almost surely continuous at any fixed value of $\phi,$ and therefore
the random process $\phi\mapsto \Lambda_{\ell\to u ,\phi}$  is continuous in distribution.  Therefore, the
random variables $\eexp^{Z_{0,\phi}^{t^*+1}}$ can be constructed on a single probability space for $0 \leq \phi \leq 1$ to
form a martingale which is continuous in distribution.    Since $b_{t^*+1,\phi}$  is the expectation of a bounded, continuous, convex
function of $\eexp^{Z_{0,\phi}^{t^*+1}}$, it follows that $b_{t^*+1,\phi}$ is continuous and nondecreasing in $\phi.$
Therefore, we can conclude that  there exists a value of $\phi$ so that $b_{t^*+1,\phi} =  \frac{\nu}{2(C-\lambda)},$
as claimed.

 Since there is no overshoot, we obtain as before (by using \prettyref{eq:aphi} for $t=t^*+1$ to  modify \prettyref{lmm:btat}
to handle $(b_{t+1}, b_{t})$ replaced by  $(b_{t^*+2,\phi}, b_{t^*+1,\phi})$):
$$
b_{t^*+2,\phi} \geq \exp(\lambda b_{t^*+1,\phi})\left(1-\eexp^{-\nu/2} \right) = \exp(\lambda\nu/(2(C-\lambda))\left(1-\eexp^{-\nu/2}  \right).
$$
The same martingale argument used in the previous paragraph can be used to show that  $b_{t^*+2,\phi}$ is nondecreasing in $\phi,$  and in particular,
$b_{t^*+2} =b_{t^*+2,1} \geq b_{t^*+2,\phi}$ for $0 \leq \phi \leq 1.$   Hence, by \prettyref{lmm:info_monotone}
and the fact $t^*+2 \leq \bar t_0 + \log^*(\nu) +2,$ we have
$b_{ \bar t_0 + \log^*(\nu) +2 } \geq  b_{ t^* +2 } \geq b_{ t^* +2,\phi},$ completing the proof of the lemma.
\end{proof}

\begin{lemma} \label{lmm:lowerboundexp}
Let $B = (p/q)^{3/2}$. Then 
\begin{align*}
\exp \left(  - \frac{\lambda}{8} b_{t} \right)  \le \expect{\eexp^{Z_0^{t+1} /2} } \le \exp \left(  - \frac{\lambda}{8B} b_{t} \right).
\end{align*}
\end{lemma}
\begin{proof}
We prove the upper bound first.
In view of \prettyref{lmm:BPTree}, by defining $f(x)=\frac{ x (p/q) +1 }{x +1}$, we get that
\begin{align*}
\eexp^{ \Lambda_u^{t+1}/2 }   = \eexp^{- K(p-q)/2 }
\prod_{ \ell \in \partial u }  f^{1/2} \left(  \eexp^{ \Lambda^{t}_{\ell \to u} - \nu}   \right).
\end{align*}
Thus,
\begin{align*}
\expect{\eexp^{ Z^{t+1}_0 /2 } }= \eexp^{- K(p-q)/2 }
  \expect{ \left( \expect{ f^{1/2} \left( \eexp^{ Z^{t}_1 - \nu }   \right) } \right )^{L_u}}
 \expect{ \left( \expect{ f^{1/2}\left(  \eexp^{ Z^{t}_0 - \nu }  \right) } \right)^{M_u}}.
\end{align*}
Using the fact that $\expect{c^X} = \eexp^{\lambda (c-1) }$ for $X \sim \Pois(\lambda)$ and $c>0$, we have
\begin{align}  
& \expect{\eexp^{ Z^{t+1}_0 /2 } } = \exp \left[ -  K(p-q)/2 + Kq\left(  \expect{ f^{1/2}  \left( \eexp^{ Z^{t}_1 - \nu }   \right) }  -1 \right) \right. \label{eq:battacharyupper}   \\
&~~~~~~~  \left. + (n-K) q\left( \expect{  f^{1/2 }\left(  \eexp^{ Z^{t}_0 - \nu }  \right)  } -1 \right)   \right]   \nonumber
\end{align}
By the intermediate value form of Taylor's theorem, for any $x\geq 0$ there exists $y$ with $1 \leq y \leq x$ such that
$
\sqrt{1+x} = 1 + \frac{x}{2} - \frac{x^2}{8(1+y)^{3/2}}.
$
Therefore,
\begin{align}
\sqrt{1+x} \leq 1 + \frac{x}{2} -  \frac{x^2}{8(1+A)^{3/2}} , \quad \forall 0 \leq x \leq A. \label{eq:taylor1plusx}
\end{align}
Letting $A \triangleq \frac{p}{q}-1$ and noting that $B = (1+A)^{3/2}$, we have
\begin{align*}
 \left( \frac{ \eexp^{z - \nu } (p/q)  + 1  }{1+ \eexp^{z - \nu} } \right)^{1/2}  &= \left(  1+  \frac{ p/q  -1  }{1+ \eexp^{- z +\nu} } \right)^{1/2} \\
& \le  1+  \frac{1}{2} \frac{ (p/q-1)}{ (1+ \eexp^{- z+\nu }) } - \frac{1}{8B}   \frac{ (p/q-1)^2}{ \left( 1+ \eexp^{- z +\nu } \right)^2  } .
\end{align*}
It follows that
\begin{align*}
& Kq\left(  \expect{ f^{1/2}  \left( \eexp^{ Z^{t}_1 - \nu }   \right) }  -1 \right)
+ (n-K) q\left( \expect{  f^{1/2 }\left(  \eexp^{ Z^{t}_0 - \nu }  \right)  } -1 \right)  \\
\le  & \frac{1}{2} K q (p/q-1) \left(  \expect{ \frac{ 1 } {1+ \eexp^{-  Z_1^t +\nu } } }    + \eexp^{\nu} \expect{ \frac{1 } {1+ \eexp^{-  Z_0^t +\nu } } }   \right) \\
& - \frac{1}{8B}  K q (p/q-1)^2 \left(  \expect{  \frac{ 1 } { ( 1+ \eexp^{-  Z_1^t +\nu } )^2  }   } + \eexp^{\nu}  \expect{  \frac{ 1 } { ( 1+ \eexp^{-  Z_0^t +\nu } )^2  }   }  \right) \\
= &K (p-q)/2 - \frac{1}{8B} K q (p/q-1)^2  \expect{  \frac{ 1 } { 1+ \eexp^{-  Z_1^t +\nu }  }  }  \\
=& K (p-q)/2 - \frac{\lambda}{8B} \underbrace{\expect{  \frac{ \eexp^{  Z_1^t } } { 1+ \eexp^{ Z_1^t - \nu }  }  }}_{b_t},
\end{align*}
where the first equality follows from \prettyref{eq:symmetry1} and \prettyref{eq:symmetry2};
the last equality holds due to $Kq (p/q-1)^2 \eexp^{\nu} = \lambda$. Combining the last displayed equation with \prettyref{eq:battacharyupper}
yields the desired upper bound.

The proof for the lower bound is similar. Instead of \prettyref{eq:taylor1plusx}, we use the
the inequality that $\sqrt{1+x} \geq  1 + \frac{x}{2} - \frac{x^2}{8}$ for all $x\geq 0$, and the
lower bound readily follows by the same argument as above.
\end{proof}

\begin{lemma}[Upper bound on classification error for the random tree model]   \label{lmm:bp_tree_error}
Consider the random tree model with parameters $\lambda,$ $\nu,$  and $p/q.$
Let $\lambda$ be fixed with $\lambda > 1/\eexp.$    There are constants $\bar t_0$ and $\nu_o$ depending
only on $\lambda$ such that if $\nu \geq \nu_o$ and $\nu \geq 2(C-\lambda)$,
then after $\bar t_0 + \log^*(\nu)+2$ iterations of the belief propagation algorithm,  the average error probability
for the MAP estimator  $\hat \tau_u$ of $\tau_u$ satisfies
\begin{equation} \label{eq:p_bound_long}
p_e^t  \leq  \left( \frac{K(n-K)}{n^2} \right)^{1/2}   \exp\left(     -\frac{\lambda}{8B}\exp(\nu \lambda/(2(C-\lambda)    )    \left(1-\eexp^{-\nu/2}  \right)  \right) ,
\end{equation}
where $B=\left( \frac{p}{q} \right)^{3/2}$ and  $C=\lambda\left(\frac{p}{q}+2\right).$
In particular,  if  $p/q = O(1),$ and $r$ is any positive constant,  then
if $\nu$ is sufficiently large,
\begin{equation} \label{eq:p_bound_short}
p_e^t  \leq \frac{K \eexp^{-r\nu} }{n}    = \frac{K}{n} \left( \frac{K}{n-K} \right)^r.
\end{equation}
\end{lemma}

\begin{proof}
We use the Bhattacharyya upper bound in  \prettyref{eq:Bhatt} with $\pi_1=\frac{K}{n}$ and $\pi_0=\frac{n-K}{n}$,
and the fact $\rho =  \expect{\eexp^{Z_0^t/2}}.$
Plugging in the lower bound on $b_{\bar t_0+\log^* (\nu) + 2}$ from \prettyref{lmm:b_lower_bnd}
into the upper bound on $\expect{\eexp^{Z_0^t/2}}$ from \prettyref{lmm:lowerboundexp}
yields \prettyref{eq:p_bound_long}.   If $p/q = O(1)$ and $r > 0,$  then for
$\nu$ large enough,
$$
\frac{\lambda}{8B}\exp(\nu \lambda/(2(C-\lambda)    )    \left(1-\eexp^{-\nu/2}  \right)   \geq \nu (r+1/2),
$$
which, together with \prettyref{eq:p_bound_long}, implies \prettyref{eq:p_bound_short}.
\end{proof}

\subsection{Lower bounds on classification error for Poisson tree}  \label{sec:lower_bnd_tree}

The bounds in this section will be combined with the coupling lemmas of 
\prettyref{app:coupling_lemmas}
to yield converse results for recovering
a community by local algorithms.

\begin{lemma}[Lower bounds for Poisson tree model]  \label{lmm:lower_tree}
Fix $\lambda$ with $0 < \lambda \le 1/ \eexp$.
For any estimator  $\hat \tau_u$ of $\tau_u$ based on observation of the tree up to any depth $t,$
the average error probability satisfies
\begin{equation}  \label{eq:p_e_t_Blower_bound}
p_e^t  \geq   \frac{K(n-K) }{n^2 }   \exp \left(   -\lambda \eexp/4   \right),
\end{equation}
and the sum of Type-I and Type-II error probabilities satisfies
\begin{equation}    \label{eq:p_e_t_B_sum_lower_bound}
p_{e,0}^t + p_{e,1}^t \geq \frac{1}{2}   \exp \left(   -\lambda \eexp/4   \right).
\end{equation}
Furthermore, if $p/q=O(1)$ and $\nu\to\infty,$ then
\begin{equation}   \label{eq: p_e_lower_bound}
\liminf_{n \to \infty} \frac{n}{K} p_e^t  \geq  1.
\end{equation}
\end{lemma}

\begin{proof}
\prettyref{lmm:lowerboundexp} shows that the Bhattacharyya coefficient, given by
$\rho_B=  \eexpect{\eexp^{Z_0^{t+1} /2} },$   satisfies $\rho_B \geq \exp \left(  - \frac{\lambda}{8} b_{t} \right) .$
Note that $b_{t+1}\leq a_{t+1} = \eexp^{\lambda b_t}$ for $t\geq 0$ and
$ b_0 = \frac{1}{1+\eexp^{-\nu}}$. It follows from induction and the assumption
$\lambda e \le 1$ that $b_t \leq \eexp$ for all $t\geq 0.$
Therefore, $\rho_B \geq  \exp \left(   -\lambda \eexp/8  \right).$  Applying
the Bhattacharyya lower bound  on $p_e^t$ in \eqref{eq:Bhatt} (which holds for any estimator)
with $(\pi_0,\pi_1)=(\frac{n-K}{n}, \frac{K}{n})$
yields \eqref{eq:p_e_t_Blower_bound} and with  $(\pi_0,\pi_1)=(1/2, 1/2)$ yields \eqref{eq:p_e_t_B_sum_lower_bound},
respectively.

It remains to prove \eqref{eq: p_e_lower_bound}, so suppose $p/q=O(1)$ and $\nu\to\infty.$
It suffices to prove \eqref{eq: p_e_lower_bound} for the MAP estimator, $\hat \tau_u = \indc{\Lambda_u^t\geq \nu},$
because the MAP estimator minimizes the average error probability.
\nbnew{Second of three references to arXiv version. When type is arXiv the lemma in appendix is referenced here.}
 \ifthenelse{\equal{\type}{arXiv}}{\prettyref{lmm:treeGaussian} implies that, as $n\to \infty,$  the  Type-I and Type-II error probabilities satisfy,}
 {It is shown in the arXiv version of this paper  \cite{HajekWuXu_one_beyond_spectral15}
  that the distribution of  $\Lambda_u^{t+1}$ is well approximated
by the $\calN(\lambda b_t/2, \lambda b_t)$ distribution if $\sigma_u=1$ and by the
 $\calN(-\lambda b_t/2,\lambda b_t)$ distribution if $\sigma_u=0.$
It follows that as $n\to \infty,$  the  Type-I and Type-II error probabilities satisfy,}
$$
p^t_{e,1} -  Q\left( \frac{\lambda b_{t-1}/2 - \nu}{\sqrt{ \lambda b_{t-1}}}  \right) \to 0  ~~~\mbox{and}~~~
p^t_{e,0} -  Q\left( \frac{\lambda b_{t-1}/2+\nu}{\sqrt{\lambda b_{t-1}}} \right)  \to 0,
$$
where $Q$ is the complementary CDF of the standard normal distribution. 
Recall that $ b_t \leq \eexp$ for all $t\geq 0.$
Also, $b_t$ is bounded away from zero, because $b_t \geq b_0 =\frac{1}{1+\eexp^{-\nu}}.$
Since $\nu \to \infty$, we have that $p^t_{e,1} \to 1$. By definition, $\frac{n}{K} p_e^t \ge p_{e,1}^t$
and consequently $\liminf_{n \to \infty} \frac{n}{K} p_e^t  \geq  1.$
\end{proof}

\section{Proofs of main results of belief propagation}  \label{section:message_passing}
\begin{proof}[Proof of \prettyref{thm:BP_Bernoulli}]



The proof basically consists of combining  \prettyref{lmm:bp_tree_error} and 
the coupling lemma \ref{lmm:treecoupling}.
\prettyref{lmm:bp_tree_error} holds by the assumptions $\frac{K^2(p-q)^2}{(n-K)q}\equiv \lambda$ for a constant
$\lambda$ with $\lambda >1/\eexp,$   $\nu \to \infty,$  and $p/q =O(1).$
\prettyref{lmm:bp_tree_error} also determines the given expression for $t_f.$   In turn, the assumptions $(np)^{\log^* \nu}  = n^{o(1)}$  and $e^{\log^*\nu} \le \nu= n^{o(1)} $
ensure that $(2+np)^{t_f}=n^{o(1)},$
so that \prettyref{lmm:treecoupling} holds. 

A subtle point is that  the performance bound of \prettyref{lmm:bp_tree_error} is for the
MAP rule \prettyref{eq:map_rule} for detecting the label of the root  vertex.   The same rule could be implemented
at each vertex of the graph $G$ which has a locally tree like neighborhood of radius $t_0+ \log^* (\nu) +  2$  by using
the estimator $\hat C_o = \{ i : R_{i}^{t_f} \geq \nu \}.$   We first bound the performance for $\hat C_o $
and then do the same for $\hat C$ produced by Algorithm \ref{alg:MP_commun}.
(We could have taken $\hat C_o$ to be the output of Algorithm \ref{alg:MP_commun}, but returning
a constant size estimator leads to simpler analysis of the algorithm for exact recovery.)

The average probability of misclassification of any given vertex $u$ in $G$  by $\hat C_o$  (for
prior distribution  $(\frac{K}{n},\frac{n-K}{n})$)  is less than or equal to the sum of two
terms.   The first term is  $n^{-1 + o(1)}$ in case $|C^*|\equiv K$ or  $n^{-1/2 + o(1)}$ in the other case
(due to failure of tree coupling of radius $t_f$ neighborhood--see \prettyref{lmm:treecoupling}).
The second term is  $\frac{K}{n} \eexp^{-\nu r}$  (bound on average error probability for the detection
problem associated with a single vertex $u$ in the tree model--see \prettyref{lmm:bp_tree_error}.)
Multiplying by $n$ bounds the expected total number of misclassification errors,  $ \expect{  |C^* \triangle \hat{C}_o|  };$
dividing by $K$ gives the bounds stated in the lemma with $\hat C$ replaced by $\hat C_o$ and the factor 2 dropped in
the bounds.

The set $\hat C_o$ is defined by a threshold condition whereas $\hat C$ similarly corresponds to using a
data dependent threshold and tie breaking rule to arrive at $|\hat C| \equiv K.$     Therefore, with probability
one, either $\hat C_o \subset \hat C$ or $\hat C \subset \hat C_o.$  Together with the fact $|\hat C| \equiv K$
we have
$$
|C^* \triangle \hat C | \leq | C^*\triangle \hat C_o | +   | \hat C_o \triangle \hat C | =  | C^*\triangle \hat C_o |  +| | \hat C_o | - K|,
$$
and furthermore,
$$
| | \hat C_o | - K| \leq | | \hat C_o | - |C^*| |  +| |C^*| - K| \leq   | C^*\triangle \hat C_o | +  | |C^*| - K| .
$$
So
$$
|C^* \triangle \hat C | \leq  2   | C^*\triangle \hat C_o |   + \|C^*| - K|.
$$
If $|C^*| \equiv K$ then $|C^* \triangle \hat C | \leq  2   | C^*\triangle \hat C_o |$ and \prettyref{eq:BP_weak_fixed_size} follows
from what was proved for $\hat C_o.$
In the other case, $\expect{ \|C^*| - K| } \leq n^{\frac{1}{2} + o(1)},$ and \prettyref{eq:BP_weak_random_size} follows
from what was proved for $\hat C_o.$

As for the computational complexity guarantee, notice that in each BP iteration, each vertex $i$ needs to
transmit the outgoing message $R_{i \to j}^{t+1}$ to its neighbor $j$ according to \prettyref{eq:mp_commun}.
To do so, vertex $i$ can first compute $R_{i}^{t+1}$ and then subtract neighbor $j$'s contribution
from it to get the desired message $R_{i \to j}^{t+1}$. In this way, each vertex $i$ needs $O(| \partial i| )$ basic
operations and the total time complexity of one BP iteration is $O( |E(G)| )$, where $|E(G)|$ is the total number of edges.
Since $\nu \leq n,$ at most $O(\log^* n)$ iterations are needed and hence the algorithm terminates in $O( |E(G)| \log^\ast n)$ time.
\end{proof}

\begin{proof}[Proof of \prettyref{thm:MP_plus_Bernoulli}]
The theorem follows from the fact that the belief propagation algorithm achieves
weak recovery, even if the cardinality $|C^*|$ is random and is only known to satisfy
$\prob{  |~ |C^*| - K| \geq \sqrt{3K\log n}  }  \leq n^{-1/2+o(1)}$ and the results in  \cite{HajekWuXu_one_info_lim15}.
We include the proof for completeness.
Let $C_k^\ast = C^\ast \cap ([n]\backslash S_k)$ for $1\le k \le 1/\delta$.
As explained in \prettyref{rmk:C_bound}, $C_k^\ast$ is obtained by sampling the
vertices in $[n]$ without replacement, and thus the distribution of $C_k^\ast$ is hypergeometric
with $\expect{ |C_k^\ast| } = K (1-\delta)$.  A result of  Hoeffding \cite{Hoeffding63} implies that the
Chernoff bounds for the  $\Binom\left(n(1-\delta), \frac{K}{n} \right) $  distribution  also hold for $|C_k^*|,$
so \prettyref{eq:Chernoff1} and \prettyref{eq:Chernoff2} with $np=K(1-\delta)$ and $\epsilon=\sqrt{ 3\log n/ [K(1-\delta)]  }$
imply $$\prob{  \big| |C_k^*| - K(1-\delta) \big| \geq \sqrt{3K(1-\delta) \log n}  }  \leq  2 n^{-1}\leq n^{-1/2+o(1)}.$$
Hence, it follows from \prettyref{thm:BP_Bernoulli}  and the condition $\lambda >1/\eexp$ that
\begin{align*}
\prob{ | \hat{C}_k \Delta C_k^\ast | \le \delta K \text{ for } 1\le k \le 1/\delta } \to 1,
\end{align*}
as $n \to \infty$, where $\hat{C}_k$ is the output
of the BP algorithm in Step 3 of  \prettyref{alg:MPplus_Bernoulli}.
Applying \cite[Theorem 3]{HajekWuXu_one_info_lim15} together with assumption \prettyref{eq:planted_dense_exact_suff1XX}, we get that  $\Prob\{\tilde{C} = C^*\} \to 1$  as $n \to \infty.$
\end{proof}

\begin{proof}[Proof of \prettyref{thm:planted_BP_converse}]
The average error probability, $p_{e}$, for classifying the label of a vertex in the
graph $G$ is greater than or equal to the lower bound \eqref{eq:p_e_t_Blower_bound}
on average error probability for the tree model, minus the upper bound, $n^{-1+o(1)},$ on
the coupling error provided  by \prettyref{lmm:treecoupling}.
Multiplying the lower bound on average error probability per vertex by $n$  yields
\eqref{eq:lower_bnd_comm_Bhat}.   Similarly,
$p_{e,0}$  and $p_{e,1},$ for the community recovery  problem can be approximated by
the respective conditional error probabilities for the random tree model by the last part of the
coupling lemma, \prettyref{lmm:treecoupling},  so
\eqref{eq:Psucc_bnd} follows from \eqref{eq:p_e_t_B_sum_lower_bound}.

By \prettyref{lmm:lower_tree},  assuming $p/q =O(1)$ and $\nu\to \infty$,
$\liminf_{n\to\infty} \frac{n}{K} \tilde p_e^t \geq 1,$  where $\tilde p_e^t$ is the
average error probability for any estimator for the corresponding random tree network.   By the
coupling lemma, \prettyref{lmm:treecoupling},
$| \tilde p_e^t -  p_e^t |\leq n^{-1+o(1)}.$
By assumption that $\frac{n}{K}=n^{o(1)}$, $| \frac{n}{K} \tilde p_e^t -  \frac{n}{K} p_e^t |  \leq  n^{-1+o(1)}.$
The conclusion $ \liminf_{n\to\infty}  \frac{n}{K}p_e \geq 1$ follows from the triangle inequality.
\end{proof}


%
%
%
%
%

\begin{appendix}

\section{Degree-thresholding when $K\asymp n$}\label{app:degreethreshold}
A simple algorithm for recovering $C^\ast$ is degree-thresholding. Specifically, let $d_i$ denote the degree of vertex $i$. Then $d_i$ is distributed
as the sum of two independent random variables, with distributions $\Binom(K-1,p)$ and $\Binom(n-K, q),$ respectively,
if $i \in C^\ast$, while $d_i \sim \Binom(n-1, q)$ if $i \notin C^\ast$.
The mean degree difference between these two distributions is $(K-1)(p-q)$, and the degree variance is $O(nq)$.
By assuming $p/q$ is bounded, it follows from the Bernstein's inequality that
$|d_i -\Expect[d_i]| \ge (K-1)(p-q)/2$ with probability at most $\eexp^{-\Omega( (K-1)^2 (p-q)^2 / (nq ) ) }$.
Let  $\hat{C}$ be the set of vertices with degrees larger than
$nq+(K-1)(p-q)/2$ and thus $\eexpect{ |\hat C  \triangle C^*| } = n \eexp^{-\Omega( (K-1)^2 (p-q)^2 / (nq ) ) }.$
Hence, if $ (K-1)^2 (p-q)^2/ (n q)  = \omega( \log \frac{n}{K})$,
then $\eexpect{ |\hat C  \triangle C^*| } =o(K)$, i.e., weak recovery is achieved.
In the regime $K \asymp n-K \asymp n$ and $p$ is bounded away from $1$,  the necessary and sufficient condition for the existence of estimators providing weak recovery, 
is $K^2 (p-q)^2/(nq) \to \infty$ as shown in \cite{HajekWuXu_one_info_lim15}.
Thus, degree-thresholding provides weak recovery in this regime whenever it is information theoretically possible.
Under the additional condition \prettyref{eq:planted_dense_exact_suff1XX}, an algorithm attaining exact recovery can be built using degree-thresholding
for weak recovery followed by a linear time voting procedure,  as in  \prettyref{alg:MPplus_Bernoulli} (see  \cite[Theorem 3]{HajekWuXu_one_info_lim15} and its proof). In the regime $\frac{n}{K} \log \frac{n}{K} = o(\log n)$, or equivalently $ K=\omega( n \log \log n/ \log n)$, the information-theoretic necessary condition for exact recovery given by  \prettyref{eq:MLE_comm_nec_cond2} and
\prettyref{eq:planted_dense_exact_necc1XX} imply that $ K^2 (p-q)^2/ (n q)  = \omega( \log \frac{n}{K})$, and hence in this regime  the degree-thresholding attains exact recovery
whenever it is information theoretically  possible.

\section{Comparison with information theoretic limits} \label{app:comparsion_info}
As noted in the introduction, in the regime $K=\Theta(n)$, degree-thresholding achieves weak recovery and, if a voting procedure is also used, exact recovery whenever it is information theoretically possible.
This section compares the recovery thresholds by belief propagation to the information-theoretic thresholds established in \cite{HajekWuXu_one_info_lim15},  in the regime of
\begin{equation}
K=o(n), \quad np  = n^{o(1)},  \quad p/q = O(1),
  \label{eq:focus}
\end{equation}
which is the main focus of this paper.

The information-theoretic threshold for weak recovery is established in \cite[Corollary 1]{HajekWuXu_one_info_lim15}, which, in the regime \prettyref{eq:focus}, reduces to the following:
If
\begin{align}
\liminf_{n \to \infty} \frac{K d(p\| q) }{2 \log \frac{n}{K} } >1, \label{eq:MLE_comm_suff_cond2}
\end{align}
then weak recovery is possible. On the other hand, if weak recovery is possible, then
\begin{align}
\liminf_{n \to \infty} \frac{K d(p\| q) }{2 \log \frac{n}{K} } \ge 1. \label{eq:MLE_comm_nec_cond2}
\end{align}
To compare with belief propagation, we rephrase the above sharp threshold in terms of the signal-to-noise ratio $\lambda$ defined in \prettyref{eq:lambda}.
Note that $d(p\|q) = ( p\log \frac{p}{q} + q - p) (1+o(1))$ provided that $p/q=O(1)$ and $p \to 0$. Therefore the information-theoretic weak recovery threshold is given by
\begin{equation}
  \lambda > (C(p/q) + \epsilon) \frac{K}{n} \log \frac{n}{K},
  \label{eq:lambda-weak}
\end{equation}
for any $\epsilon > 0$, where $C(\alpha) \triangleq \frac{2 (\alpha-1)^2}{1 - \alpha + \alpha \log \alpha}$.
In other words, in principle weak recovery only demands a vanishing signal-to-noise ratio $\lambda = \Theta(\frac{K}{n} \log \frac{n}{K})$, while, in contrast, belief propagation requires $\lambda > 1/e$ to achieve weak recovery.
No polynomial-time algorithm is known to succeed for $\lambda \leq 1/e$, suggesting that
computational complexity constraints might incur a severe penalty on the statistical optimality in the sublinear regime of $K=o(n)$.

\medskip
Next we turn to exact recovery.
The information-theoretic optimal threshold has been established in \cite[Corollary  3]{HajekWuXu_one_info_lim15}.
In the regime of interest \prettyref{eq:focus},
exact recovery is possible via the maximum likelihood estimator (MLE) provided that \prettyref{eq:MLE_comm_suff_cond2} and \prettyref{eq:planted_dense_exact_suff1XX} hold.
Conversely, if exact recovery is possible, then \prettyref{eq:MLE_comm_nec_cond2} and
\begin{align}
\liminf_{n \to \infty}  \frac{ K  d(\tau^\ast \| q) }{\log n } \ge 1 \label{eq:planted_dense_exact_necc1XX}
\end{align}
must hold.
Notice that the information-theoretic sufficient condition for exact recovery has two parts: one is the information-theoretic sufficient
condition \prettyref{eq:MLE_comm_suff_cond2} for weak recovery; the other is the sufficient condition  \prettyref{eq:planted_dense_exact_suff1XX}
for the success of the linear time voting procedure.   Similarly, recall that the sufficient condition for exact recovery by belief propagation also has two parts:
one is the sufficient condition $\lambda >1/\eexp$ for weak recovery, and the other is again   \prettyref{eq:planted_dense_exact_suff1XX}.

Clearly, the information-theoretic  sufficient conditions for exact recovery and $\lambda  >1/\eexp$, which is needed for weak recovery by local algorithms, are both at least as strong as the information theoretic
necessary conditions \prettyref{eq:MLE_comm_nec_cond2} for weak recovery.
It is thus of interest to compare them by assuming that \prettyref{eq:MLE_comm_nec_cond2} holds.
If $p/q$ is bounded, $p$ is bounded away from $1$, and  \prettyref{eq:MLE_comm_nec_cond2} holds,
then  $d(\tau^\ast \|q) \asymp d(p\|q) \asymp \frac{(p-q)^2}{q}$ as shown in \cite{HajekWuXu_one_info_lim15}. So under those conditions on $p, q$ and \prettyref{eq:MLE_comm_nec_cond2},
and if $K/n$ is bounded away from $1$,
\begin{align}   \label{eq:exact_vs_poly_Bernoulli_XX}
\frac{Kd(\tau^\ast \| q)}{\log n } \asymp  \frac{K(p-q)^2}{q \log n } \asymp  \left(   \frac{n }{K \log n} \right) \lambda.
\end{align}
Hence, the information-theoretic sufficient condition for exact recovery \prettyref{eq:planted_dense_exact_suff1XX}
demands a signal-to-noise ratio
\begin{equation}
\lambda = \Theta\pth{\frac{K \log n}{n}}.
  \label{eq:snr-exact}
\end{equation}

Therefore, on one hand, if $K=\omega(n/\log n)$,
then condition \prettyref{eq:planted_dense_exact_suff1XX} is stronger than $\lambda >1/\eexp$, and
thus  condition \prettyref{eq:planted_dense_exact_suff1XX} alone
is sufficient for local algorithms to attain exact recovery.
On the other hand, if $K=o(n/\log n)$, then $\lambda>1/ \eexp$ is stronger
than condition \prettyref{eq:planted_dense_exact_necc1XX}, and thus for local algorithms to achieve exact recovery, it requires $\lambda > 1/e$, which far exceeds the
information-theoretic optimal level \prettyref{eq:snr-exact}.
The critical value of $K$ for this crossover is  $K=\Theta\left(\frac{n}{\log n} \right).$
To determine the precise crossover point, we solve for $K^*$ which satisfies
\begin{align}
\frac{ K  d(\tau^\ast \| q) }{\log n } & =1,  \label{eq:exact_threshold_equ} \\ 
\lambda = \frac{K^2(p-q)^2}{nq} & = \frac{1}{e}. \label{eq:bp_threshold_equ}
\end{align}
Let $c=p/q=O(1).$ It follows from \prettyref{eq:bp_threshold_equ} that 
\begin{align}
q = \frac{n}{K^2 (c-1)^2 e }. \label{eq:q_expression}
\end{align}
Plugging \prettyref{eq:q_expression} into the definition of $\tau^*$ in \prettyref{eq:deftau}, we get that
\begin{align*}
\tau^\ast = \left(1+o(1) \right) q \frac{ c-1  }{ \log c  } .
\end{align*}
It follows that 
$$
d \left(\tau^* \| q \right) = \left(1+o(1) \right) q \left( 1 - \frac{c-1}{\log c} \log \frac{e \log c}{c-1} \right).
$$
Combining the last displayed equation with  \prettyref{eq:exact_threshold_equ} and \prettyref{eq:q_expression}
yields the crossover point $K^*$ given by
$$
K^* = \frac{n}{\log n} \left(\rho_{\sf BP}(c) + o(1) \right),
$$ 
where 
$$
\rho_{\sf BP}(c) = \frac{1}{e (c-1)^2} \left({1 - \frac{c-1}{\log c} \log \frac{e \log c}{c-1} } \right).
$$
\prettyref{fig:exact_phase_diagram} shows the phase diagram with $K=\rho n /\log n$ for a fixed constant $\rho$. 
The line $\{ (\rho, \lambda): \lambda = 1/e\}$ corresponds to the
weak recovery, while the line $\{(\rho,\lambda): \lambda = \rho/( e \rho_{\sf BP} ) \}$
corresonds to the information-theoretic exact recovery threshold. Therefore, 
BP plus voting (\prettyref{alg:MPplus_Bernoulli}) achieves optimal exact recovery whenever the former line lies below the latter, or equivalently, $\rho>\rho_{\sf BP}(c) ).$ 

\section{Coupling lemma}  \label{app:coupling_lemmas}  

Consider a sequence of planted dense subgraph models  $G=(E,V)$ as described in the introduction.  
For each $i\in V,$  $\sigma_i$ denotes the indicator of $i\in C^*.$     For $u \in V,$
let $G_u^t$ denote the subgraph of $G$ induced by the vertices whose distance from $u$ is at most $t.$
Recall from \prettyref{sec:BP_Bernoulli_tree} that $T_u^t$ is defined similarly for the random tree graph, and
$\tau_i$ denotes the label of a vertex $i$ in the tree graph.
The following lemma  shows there is a coupling such that
$\left(G_u^{t_f}, \sigma_{G_u^{t_f}}\right) = \left(T_u^{t_f}, \tau_{T_u^{t_f}}\right)$ with probability converging to $1$,
where $t_f$ is growing slowly with $n.$
A version of the lemma for fixed $t$, assuming $p, q=\Theta(1/n)$ is proved in
 \cite[Proposition 4.2]{Mossel12}, and the argument used there extends to prove this version.
 \ifthenelse{\equal{\type}{arXiv}}{}{The proof is provided in the arXiv version of this paper  \cite{HajekWuXu_one_beyond_spectral15}.}
\nbnew{Third of three references to arXiv version.}

\begin{lemma}[Coupling lemma]   \label{lmm:treecoupling}
Let $d=np.$
Suppose $p, q, K$ and  $t_f$  depend on $n$  such that $t_f$ is positive integer valued,
and $(2+d)^{t_f }= n^{o(1)}.$
Consider an instance of the planted dense subgraph model.
Suppose that $C^*$ is random and all $\binom{n}{|C^*|}$ choices of $C^*$ are equally
likely give its cardinality, $|C^*|.$  (If this is not true, this lemma still applies to the random graph obtained by randomly,
uniformly permuting the vertices of $G$.)
If the planted dense subgraph model (\prettyref{def:pds_model}) is such that $|C^*|\equiv K,$ then  for any fixed $u \in [n]$, there exists a coupling between
$(G, \sigma )$ and $(T_u, \tau_{T_u})$ such that
\begin{align}    \label{eq:TV_convergence_det}
\prob{\left( G_u^{t_f}, \sigma_{G_u^{t_f}}\right) = \left(T_u^{t_f}, \tau_{T_u^{t_f}} \right) } \ge 1 -  n^{-1+o(1) }.
\end{align}
If the planted dense subgraph model is such that $|C^*|\sim \Binom(n,K/n),$ then  for any fixed $u \in [n]$, there exists a coupling
between $(G, \sigma )$ and $(T_u, \tau_{T_u})$ such that
\begin{align}    \label{eq:TV_convergence}
\prob{\left( G_u^{t_f}, \sigma_{G_u^{t_f}}\right) = \left(T_u^{t_f}, \tau_{T_u^{t_f}} \right) } \ge 1 -  n^{-1/2+o(1) }.
\end{align}
If the planted dense subgraph model  is such that $K \geq 3\log n$ and $|C^*|$ is random such that
$\prob{  | |C^*| - K| \geq \sqrt{3K\log n}  }  \leq n^{-1/2+o(1)}$,   then there exists a coupling
between $(G, \sigma )$ and $(T_u, \tau_{T_u})$ such that \prettyref{eq:TV_convergence} holds.

Furthermore, the bounds stated remain true if the label, $\sigma_u$, of the vertex $u$ in the planted community graph, and
the label $\tau_u$ of the root vertex in the tree graph, are both conditioned to be 0 or are both conditioned to be one.
\end{lemma}

\begin{remark}   The condition $(2+d)^{t_f }= n^{o(1)}$ in \prettyref{lmm:treecoupling} is satisfied, for example, if
$t_f = O(\log^* n)$ and $d\leq n^{o(1/\log^* n)},$    or if
$t_f=O(\log\log n)$ and $d= O((\log n)^s )$  for some constant $s>0.$     In particular, the condition is satisfied if
$t_f = O(\log^* n)$  and  $d= O((\log n)^s )$  for some constant $s>0.$
\end{remark}

\arxivonly{
\begin{remark}  The condition $(2+d)^{t_f }= n^{o(1)}$ is equivalent to $(a+bd)^{t_f } = n^{o(1)}$ for any constants
$a$ and $b$ with $a>1$ and $b>0.$  Also, if $d \geq 1+\epsilon$ for all $n$ and some fixed constant $\epsilon >0,$ the
condition is equivalent to $d^{t_f } =  n^{o(1)}.$
\end{remark}
}

\begin{remark}   \label{rmk:C_bound}
The part of \prettyref{lmm:treecoupling} involving $ | |C^*| - K| \geq \sqrt{3K\log n} $
 is included to handle the case that $|C^*|$ has a certain hypergeometric
distribution.  In particular,
if we begin with the planted dense subgraph model (\prettyref{def:pds_model})  with $n$ vertices and a planted dense community with $|C^*| \equiv K,$  for a
cleanup procedure we will use for exact recovery (See \prettyref{alg:MPplus_Bernoulli}),
we need to withhold a small fraction $\delta$ of vertices and run the belief propagation algorithm on
the subgraph induced by the set of $n(1-\delta)$ retained vertices.     Let $C^{**}$ denote the intersection of $C^*$ with the set of $n(1-\delta)$
retained vertices.  Then $|C^{**}|$ is obtained by sampling the
vertices of the original graph without replacement. Thus, the distribution of $|C^{**}|$ is
hypergeometric, and $\expect{|C^{**}|}=K(1-\delta).$
Therefore, by a result of Hoeffding \cite{Hoeffding63}, the distribution of
$|C^{**}|$ is convex order dominated by the distribution that would result by sampling with replacement, namely, by
$\Binom\left(n(1-\delta), \frac{K}{n}\right).$
That is, for any convex function $\Psi,$  $\expect{ \Psi( |C^{**}| )} \leq \expect{ \Psi(  \Binom(n(1-\delta), \frac{K}{n}) )  }.$
Therefore, Chernoff bounds for  $\Binom(n(1-\delta), \frac{K}{n}) ) $ also hold for $|C^{**}|.$
We use the following Chernoff bounds
for binomial distributions \cite[Theorem 4.4, 4.5]{Mitzenmacher05}: For $X \sim \Binom(n,p)$:
\begin{align}
\prob{ X \ge (1+\epsilon) n p } \le \eexp^{-\epsilon^2 n p /3}, \quad \forall 0\le \epsilon \le 1    \label{eq:Chernoff1}   \\
\prob{ X \le (1-\epsilon) n p } \le \eexp^{-\epsilon^2 n p /2}, \quad \forall 0\le \epsilon \le 1. \label{eq:Chernoff2}
\end{align}
Thus, if $ K(1-\delta) \geq 3 \log n$,
then \prettyref{eq:Chernoff1} and \prettyref{eq:Chernoff2} with $\epsilon=\sqrt{ 3\log n/ [K(1-\delta)]  }$
imply $$\prob{  \big| |C^{**}| - K(1-\delta) \big| \geq \sqrt{3K(1-\delta) \log n}  }  \leq n^{-1}.$$
Thus, \prettyref{lmm:treecoupling} can be applied with $K$ replaced by $K(1-\delta)$.
\end{remark}

\ifthenelse{\equal{\type}{APT}}{}{

\begin{proof}
We write $V=V(G)$ and $V^t=V(G)\setminus V(G_u^t)$. Let $V^t_0$ and $V^t_1$ denote the
set of vertices $i$ in $V^t$ with $\sigma_i=0$ and $\sigma_i=1$, respectively.
For a vertex $i \in \partial G_u^t$,
let $\tilde{L}_i$ denote the number of $i$'s neighbors in $V^t_1$, and
$\tilde{M}_i$ denote the number of $i$'s neighbors in $V^t_0$.
Given $V_0^t, V_1^t$, and $\sigma_i$,     $\tilde{L}_i \sim \Binom(|V_1^t|, p)$ if $\sigma_i=1$ and $\tilde{L}_i \sim \Binom(|V_1^t|, q)$ if $\sigma_i=0,$
and $\tilde{M}_i \sim \Binom(|V_0^t|, q)$ for either value of $\sigma_i.$
Also, $\tilde{M}_i$ and $\tilde{L}_i$ are  independent.

Let $C^t$ denote the event
\begin{align*}
C^t=\{ | \partial G^s_u | \le 4(2+2d)^s \log n, \forall 0 \le s \le t  \}.
\end{align*}
The event $C^t$ is useful to ensure that $V^t$ is large enough so that the binomial random variables $\tilde{M}_i$ and $\tilde{L}_i$
can be well approximated by Poisson random variables with the appropriate means.
The following lemma shows that $C^t$ happens with high probability conditional on $C^{t-1}$.
\begin{lemma}\label{lmm:C}
For $t \ge 1$,
\begin{align*}
\prob{C^t | C^{t-1} } \ge 1- n^{-4/3}.
\end{align*}
Moreover,   $P(C^t)\geq  1-tn^{-4/3},$
and conditional on the event $C^{t-1}$, $|G_u^{t-1} | \le 4(2+2d)^t \log n.$
\end{lemma}
\begin{proof}
Conditional on $C^{t-1}$, $|\partial G^{t-1}_u| \le 4(2+2d)^{t-1}  \log n $. For any $i \in \partial G^{t-1}_u$,
$\tilde{L}_i+\tilde{M}_i$ is stochastically dominated by $\Binom(n, d/n)$, and $\{ \tilde{L}_i, \tilde{M}_i\}_{i\in \partial G_u ^{t-1}}$  are independent.
It follows that $| \partial G^{t}_u|$ is stochastically dominated by (using $d+1\geq d$):
\begin{align*}
X \sim \Binom \left( 4(2+2d)^{t-1} n \log n, (d+1) /n \right).
\end{align*}
Notice that $\expect{X} = 2(2+2d)^t  \log n \ge 4 \log n.$
Hence,  in view of the Chernoff bound \prettyref{eq:Chernoff1} with $\epsilon=1,$
\begin{align*}
\prob{C^t | C^{t-1} }& \ge \prob{ X \le 4(2+2d)^t \log n } \\
& = 1- \prob{X > 2 \expect{X} } \ge 1- \eexp^{-\expect{X}/3} \ge 1-n^{-4/3}.
\end{align*}
Since $C^0$ is always true, $P(C^t)\geq (1- n^{-4/3})^t \geq 1-tn^{-4/3}.$
Finally, conditional on $C^{t-1}$,
\begin{align*}
|G_u^{t-1} | & = \sum_{s=0}^{t-1} \partial G_u^{s} \le \sum_{s=0}^{t-1} 4(2+2d)^s \log n \\
& = 4 \frac{(2+2d)^t-1 }{1+2d }  \log n \le 4(2+2d)^t  \log n.
\end{align*}
\end{proof}

Note that it is possible to have $i, i' \in \partial G_u^t$ which share a neighbor
in $V^t$, or  which themselves are connected by an edge, so $G_u^t$ may not be a tree. The next lemma shows that with high
probability such events don't occur.   For any $t \ge 1$, let $A^t$ denote the event that no vertex in $V^{t-1}$ has more than
one neighbor in $G^{t-1}_u$;  $B^t$ denote the event that there are no edges within $\partial G_u^{t}$. Note that if $A^s$
and $B^s$ hold for all $s=1, \ldots, t$, then $G_u^{t}$ is a tree.
\begin{lemma}\label{lmm:AB}
For any $t$ with $1\leq t \leq t_f,$
\begin{align*}
\prob{A^t | C^{t-1} } &\ge 1-  n^{-1+o(1) } \\
\prob{B^t | C^t } &\ge 1-  n^{-1+o(1) }.    
\end{align*}
\end{lemma}
\begin{proof}
For the first claim, fix any $i, i' \in \partial G_u^{t-1}$. For any $j \in V^{t-1}$, $
\prob{A_{i j} = A_{i', j} =1} \le d^2/n^2. $
Since $|V^{t-1} | \le n$ and conditional on $C^{t-1}$, $|\partial G^{t-1}_u| \le 4(2+2d)^{t-1} \log n  =  n^{o(1)}.$
It follows from the union bound that, given $C^{t-1},$
\begin{align*}
\prob{ \exists i, i' \in \partial G_u^{t-1}, j \in V^{t-1} : A_{i j} = A_{i', j} =1 } & \le n 16(2+2d)^{2t-2} \log^2 n \times \frac{d^2}{n^2} \\
& =  n^{-1+ o(1) }.
\end{align*}
Therefore, $\prob{A^t | C^{t-1} } \ge 1-  n^{-1+o(1)}$. For the second claim, fix any $i, i' \in \partial G_u^t$. Then $\prob{A_{i,i'} =1 } \le d/n$.
It follows from the union bound that, given  $C^t,$
\begin{align*}
\prob{ \exists i, i' \in \partial G_u^{t}: A_{i i'}=1 } \le  16(2+2d)^{2t} \log^2n  \times \frac{d}{n} \le  n^{-1+o(1) } .
\end{align*}
Therefore, $\prob{B^t | C^t } \ge 1-  n^{-1+o(1) }$.
\end{proof}
In view of Lemmas \ref{lmm:C} and \ref{lmm:AB}, in the remainder of the proof of \prettyref{lmm:treecoupling} we can
and do assume without loss of generality that $A_t, B_t, C_t$ hold for all $t \geq 0.$
We consider three cases about the cardinality of the community, $|C^*|$:
\begin{itemize}
\item   $|C^*|\equiv K.$
\item  $K \geq 3\log n$ and $\prob{  | |C^*| - K| \leq \sqrt{3K\log n}  }  \geq  1 - n^{-1/2+o(1)}.$
This includes the case that $|C^\ast| \sim \Binom(n, K/n)$ and  $K \geq 3\log n,$ as noted in \prettyref{rmk:C_bound}.
\item   $K \leq 3\log n$ and $\prob{  |C^*|  \leq  6\log n  }  \geq  1 - n^{-1/2+o(1)}.$
This includes the case that $|C^\ast| \sim \Binom(n, K/n)$ and  $K \leq 3\log n,$ because,
in this case,  $|C^\ast|$ is stochastically dominated by a $\Binom(n,3\log n / n)$ random
variable,  so  Chernoff bound \prettyref{eq:Chernoff1} with $\epsilon =1$ implies:
$\prob{   |C^\ast| \leq 6\log n }\geq 1-n^{-1}$ if $K\leq 3\log n.$
\end{itemize}
In the second and third cases we assume these bounds (i.e.,  either $ | |C^*| - K| \leq \sqrt{3K\log n}$  if  $K\geq  3\log n$
or $ |C^*|  \leq  6\log n$ if $K\leq 3\log n$) hold, without loss of generality.

We need a version of the well-known bound on the total variation distance between the binomial distribution
and a Poisson distribution with approximately the same mean:
\begin{align}
d_{\rm TV} \left( \Binom (m, p), \Pois(\mu) \right) \leq mp^2 + \psi(\mu-mp),  \label{eq:BinomPoissonApprox}
\end{align}
where $\psi(u) =  \eexp^{|u|} (1+|u|) -1.$    The term $mp^2$ on the right side of
\prettyref{eq:BinomPoissonApprox} is Le Cam's  bound on the variational distance between
the $\Binom(m,p)$ and the Poisson distribution with the same mean, $mp;$   the term
$\psi(\mu - mp)$  bounds the variational distance between the two
Poisson distributions with means $\mu$ and $mp,$  respectively (see \cite[Lemma 4.6]{Mossel12} for a proof).
Note that $\psi(u)=O(|u|)$ as $u\rightarrow 0.$

We recursively construct the coupling.
For the base case, we can arrange  that
$$
\prob{   (G_u^{0}, \sigma_{G_u^{0}}) =  (T_u^{0}, \tau_{T_u^{0}})  } =1-  |  \prob{ \sigma_u = 1 } - \prob{ \tau_u =1 }|
=1 - \bigg|  \frac{\expect{C^*}}{n}  - \frac{K}{n}  \bigg| .
$$
If $|C^*| \equiv K$ this gives $\prob{   (G_u^{0}, \sigma_{G_u^{0}}) =  (T_u^{0}, \tau_{T_u^{0}})  } =1$ and in the other
cases
$$
\prob{   (G_u^{0}, \sigma_{G_u^{0}}) =  (T_u^{0}, \tau_{T_u^{0}})  }
\geq 1 - \frac{\sqrt{3K\log n}}{n} -  n^{-1/2+o(1)} \geq 1 -  n^{-1/2+o(1)}.
$$
So fix $t \ge 1$ and assume that $(T_u^{t-1}, \tau_{T_u^{t-1}}) = (G_u^{t-1}, \sigma_{G_u^{t-1}})$.   We aim to construct
a coupling so that  $(T_u^{t}, \tau_{T_u^{t}}) = (G_u^{t}, \sigma_{G_u^{t}})$  holds with probability at least $1- n^{-1+o(1)}$ if
$|C^*|\equiv K$ and with probability at least $1-n^{-1/2 + o(1)}$ in the other cases.
Each of the vertices $i$ in $\partial G^{t-1}_u$ has a random number of neighbors $\tilde L_i$ in $V_1^{t-1}$ and a random number of neighbors
$\tilde M_i$ in $V_0^{t-1}.$   These variables are conditionally independent given $(G_u^{t-1},\sigma_{G_u^{t-1}}, |V_1^{t-1} |, |V_0^{t-1} |).$   Thus we bound the total
variational distance of these random variables from the corresponding Poisson distributions by using a union bound, summing over all
$i \in \partial G_u^{t-1}.$     Since $C^{t-1}$ holds,
$|\partial G^{t-1}_u| \leq  4(2+2d)^{t-1}  \log n  =   n^{o(1)},$ so  it suffices to show that the variational distance for the numbers of children with
each label for any given vertex in $\partial G^{t-1}_u$  is at most $n^{-1/2 + o(1)}$   (because  $n^{o(1)}n^{-1/2+o(1)}=n^{-1/2+ o(1)}).$
Specifically,  we need to obtain such a bound on the variational distances for  three types of random variables:
\begin{itemize}
\item $\tilde{L}_i$ for vertices $i\in \partial G_u^{t-1}$ with $\sigma_i=1$
\item $\tilde{L}_i$ for vertices $i\in \partial G_u^{t-1}$ with $\sigma_i=0$
\item $\tilde{M}_i$ for vertices in $i\in \partial G_u^{t-1}$ (for either $\sigma_i$) .
\end{itemize}
The corresponding variational distances, conditioned on $|V_1^{t-1}|$ and $|V_0^{t-1}|$, and the bounds
on  the distances implied by  \prettyref{eq:BinomPoissonApprox}, are as follows:
\begin{eqnarray*}
d_{TV}\left(\Binom(|V_1^{t-1}|,p),\Pois(Kp)\right) & \leq &  |V_1^{t-1}|p^2+ \psi\left((K-|V_1^{t-1}|) p\right ) \\
d_{TV}\left(\Binom(|V_1^{t-1}|,q),\Pois(Kq)\right) & \leq &  |V_1^{t-1}|q^2 + \psi\left((K-|V_1^{t-1}|)q\right)   \\
d_{TV}\left(\Binom(|V_0^{t-1}|,q),\Pois((n-K)q)\right) & \leq &  |V_0^{t-1}| q^2 + \psi\left((n-K-|V_0^{t-1}|) q \right)
\end{eqnarray*}

The assumption on $d$ implies $p \leq o(n^{-1+o(1)})$ and  $np^2  =  dp \leq  n^{-1+o(1)},$  and
thus also  $|V_1^{t-1}|q^2  \le   |V_1^{t-1}|p^2 \leq  n^{-1+o(1)}$  and $|V_0^{t-1}| q^2 \leq n^{-1+o(1)}.$
Also, for use below, $Kq^2 \leq Kp^2 \leq  n^{-1+o(1)}.$

We now complete the proof for the three possible cases concerning $ |C^\ast|.$   Consider the first
case, that $ |C^\ast| \equiv K.$     Since we are working under the assumption $C^{t-1}$ holds, in
the case $ |C^\ast| \equiv K,$
\begin{eqnarray*}
|(K-  |V_1^{t-1}| ) p  |  \leq  p | G_u^{t-1} |  \leq      p4(2+2d)^t \log n   \leq n^{-1 + o(1)}
\end{eqnarray*}
and similarly
$$
|(n-K-  |V_0^{t-1}| ) q | \leq q |G_u^{t-1} |    \leq   q 4(2+2d)^t \log n    \leq    n^{-1 + o(1)}.
$$
The conclusion \prettyref{eq:TV_convergence_det} follows, proving the lemma in case $ |C^\ast| \equiv K.$

Next consider the second case:  $ | |C^*| - K| \leq \sqrt{3K\log n}$  and  $K \geq 3\log n.$
Using $C^{t-1}$ as before, we obtain
\begin{eqnarray*}
|(K-  |V_1^{t-1}| ) p  |  \leq   \sqrt{3Kp^2 \log n} + p4(2+2d)^t \log n   \leq n^{-1/2 + o(1)}
\end{eqnarray*}
and
$$
|(n-K-  |V_0^{t-1}| ) q |   \leq   \sqrt{3Kq^2 \log n} + q 4(2+2d)^t \log n    \leq    n^{-1/2 + o(1)},
$$
which establishes \prettyref{eq:TV_convergence} in the second case.

Finally, consider the third case: $ |C^*|  \leq  6\log n$  and  $K \leq 3\log n.$   Then
\begin{eqnarray*}
|(K-  |V_1^{t-1}| ) p  | & \leq  & 6p \log n + p4(2+2d)^t \log n   \leq n^{-1/2 + o(1)}
\end{eqnarray*}
and
$$
|(n-K-  |V_0^{t-1}| ) q |   \leq   6q \log n + q 4(2+2d)^t \log n    \leq    n^{-1/2 + o(1)},
$$
which establishes \prettyref{eq:TV_convergence} in the third case.

Thus, we can construct a coupling so that  $(T_u^{t}, \tau_{T_u^{t}}) = (G_u^{t}, \sigma_{G_u^{t}})$  holds with probability
at least $1- n^{-1+ o(1)}$  in case $ |C^\ast| \equiv K,$ and with probability  $1- n^{-1/2 + o(1)}$  in the other cases,
at each of the $t_f$ steps, and, furthermore, the $o(1)$ term in  the exponents of $n$
are uniform in  $t$ over $1\leq t \leq t_f.$    Since $2^{t_f} =  n^{o(1)},$ it follows that $t_f = o(\log n).$
So the total probability of failure of the coupling is upper bounded by $t_f n^{-1 + o(1)}=  n^{-1+ o(1)}$  in
case $ |C^\ast| \equiv K$ and by $ n^{-1/2 + o(1)}$ in the other cases.

Finally, we justify the last sentence of the lemma.
At the base level of a recursive construction above, the  proof uses the fact that the labels can be coupled with high probability because
$\Prob\{\sigma_u=1\} \approx \frac{K}{n} = \Prob\{\tau_u=1\}.$   If instead we let $u$ be a vertex selected uniformly at random
from $C^*$, so that $\sigma_u\equiv 1,$ and we consider the random tree conditioned on $\tau_u=1,$ the labels of
$u$ in the two graphs are equal with probability one (i.e. exactly coupled), and then the recursive construction of the coupled
neighborhoods can proceed from there.    Similarly,  if $u$ is a vertex selected uniformly at random from $[n]\backslash  C^*,$
then the lemma goes through for  coupling with the labeled tree graph conditioned on $\tau_u=0.$    
\end{proof}

}

\ifthenelse{\equal{\type}{APT}}{\end{appendix}\end{document}}{}
\nbnew{APT version of paper ends here.}

\section{Analysis of BP on a tree continued--moments and CLT}

This section establishes messages in the BP algorithm are asymptotically Gaussian, a property
which  is used in the proof of the converse result, \prettyref{thm:planted_BP_converse}.
First bounds on the first and second moments are found and then a version of
the Berrry-Essen CLT is applied.

\subsection{First and second moments of log likelihood messages for Poisson tree}

The following lemma provides estimates for the first and second moments of the log likelihood messages
for the Poisson tree model.
\begin{lemma}  \label{lmm:BP_moments}  With $C=\lambda(p/q+2),$  for all $t\geq 0,$
\begin{align}
\expect{Z_0^{t+1} }  = - \frac{\lambda b_t }{2}    +O \left( \frac{\lambda^2\eexp^{Cb_{t-1}}   }{K(p-q)}  \right)     \label{eq:BPZ0_mean} \\
\expect{Z_1^{t+1} }  =   \frac{\lambda b_t }{2}    +O \left( \frac{\lambda^2 \eexp^{Cb_{t-1}}  }{K(p-q)} \right)      \label{eq:BPZ1_mean}   \\
\var \left(Z_0^{t+1} \right) = \lambda b_t               +O \left( \frac{\lambda^2 \eexp^{Cb_{t-1}}  }{K(p-q)} \right)       \label{eq:BPZ0_var} \\
\var \left(Z_1^{t+1} \right)  = \lambda b_t   +        O \left( \frac{\lambda^2 \eexp^{Cb_{t-1}}  }{K(p-q)} \right)         \label{eq:BPZ1_var}
\end{align}
\end{lemma}

\begin{lemma}  \label{lmm:logplus}
Let $\psi_2(x)$ and $\psi_3(x)$ be defined for $x\geq 0$ by the relations:  $\log(1+x)=x+\psi_2(x)$ and $\log(1+x)=x-\frac{x^2}{2}+\psi_3(x).$
Then $0\geq \psi_2(x) \geq  -\frac{x^2}{2},$  and $0 \leq \psi_3(x) \leq \frac{x^3}{3.}.$    In particular,
$|\psi_2(x)|\leq x^2$ and $|\psi_3(x)|\leq x^3.$    Moreover,  $|\log^2(1+x) - x^2|\leq x^3.$
\end{lemma}
\begin{proof}[Proof of \prettyref{lmm:logplus}]
By the intermediate value form of Taylor's theorem,  for any $x\geq 0,$
$
\log(1+x) =  x +  \frac{x^2}{2}\left(-\frac{1}{(1+y)^2} \right)
$
for some $y \in [0,x].$  The fact $-1 \leq -\frac{1}{(1+y)^2} \leq 0$  then establishes the
claim for $\psi_2.$    Similarly, the claim for $\psi_3$ follows from the fact that for
some $z \in [0,x]$
$
\log(1+x)   =    x   -  \frac{x^2}{2}  + \frac{x^3}{3!}    \left( \frac{2}{(1+z)^3} \right).
$
Finally, the first and second derivatives of $\log^2(1+x)$ at $x=0$ are 0 and 2, and
$$
\bigg|  \frac{1}{3!}  \left( \frac{d}{dx}\right)^3 \log^2(1 + x)   \bigg|  =\bigg| \frac{4\log(1+x) - 6}{3! (1+x)^3}\bigg|  \leq 1~~~\mbox{for}~x\geq 0,
$$
so the final claim of the lemma also follows from Taylor's theorem.
\end{proof}

\begin{proof}[Proof of \prettyref{lmm:BP_moments}]
 Plugging $g(z) = \frac{1}{ ( 1+\eexp^{-z+\nu} )^3 }$ into \prettyref{eq:changemeasure}  we have
\begin{align}
 \eexp^{\nu}  \expect{ \frac{1}{ ( 1+\eexp^{-Z_0^t+\nu} )^3 } } +\expect{ \frac{1}{ (1+\eexp^{-Z_1^t+\nu} )^3 } } & =  \expect{ \frac{1}{ (1+\eexp^{-Z_1^t+\nu} )^2 }  } \label{eq:symmetry3}.
\end{align}

Applying \prettyref{lmm:logplus}, we have
\begin{align}
\log \left(  \frac{ \eexp^{ z -\nu } (p/q)   + 1  }{ \eexp^{ z -\nu } +1} \right)
& = \log \left( 1 + \frac{p/q -1 } {  1 + \eexp^{-z+\nu}  }  \right) \label{eq:oneplus}   \\
& =  \frac{ p/q-1} {1+ \eexp^{- z +\nu}  }   -   \frac{ ( p/q-1)^2 }{ 2 ( 1+ \eexp^{ -z +\nu } )^2 } +
\psi_3  \left(   \frac{ p/q-1 }{  1+ \eexp^{- z +\nu}  }  \right). \label{eq:Taylorlog}
\end{align}
Hence,
\begin{align*}
& \Lambda^{t+1}_{u} = - K(p-q)  \\
&~~~~~~ +  \sum_{ \ell \in \partial u } \left[  \frac{ p/q-1 } {1+ \eexp^{-\Lambda^t_{\ell \to u  } +\nu } }
-   \frac{ ( p/q-1)^2 } { 2 ( 1+ \eexp^{-\Lambda^t_{\ell \to u} +\nu }  )^2 } +
\psi_3  \left(   \frac{ p/q-1 }{  1+ \eexp^{-\Lambda^t_{\ell \to u} +\nu }  }  \right)  \right].
\end{align*}

It follows, by considering the case the label of vertex $u$ is conditioned to be zero, that:
\begin{align*}
\expect{Z_0^{t+1} }  & = -K(p-q) + \expect{L_u}  \expect{  \frac{p/q-1  } {1+ \eexp^{-  Z_1^t +\nu } }  } + \expect{M_u}
\expect{  \frac{p/q-1  } {1+ \eexp^{-  Z_0^t +\nu } }  }  \\
& -  \expect{L_u}  \expect{  \frac{ (p/q-1)^2  } {2 ( 1+ \eexp^{-  Z_1^t +\nu } )^2  }   } - \expect{M_u}
\expect{  \frac{ (p/q-1)^2  } {2 ( 1+ \eexp^{-  Z_0^t +\nu } )^2  }   } \\
& + \expect{L_u}   \expect{  \psi_3 \left(          \frac{ p/q-1 }{  1+ \eexp^{-Z_1^t+\nu } }             \right)    }  + \expect{M_u}
\expect{  \psi_3 \left(  \frac{ p/q-1 }{  1+ \eexp^{-Z_0^t+\nu }  }   \right)  }.
\end{align*}
Notice that $\expect{L_u} = Kq $ and $\expect{M_u}=(n-K)q$. Thus
\begin{align*}
& \expect{L_u}  \expect{  \frac{p/q-1  } {1+ \eexp^{-  Z_1^t +\nu } }  } + \expect{M_u}
\expect{  \frac{p/q-1  } {1+ \eexp^{-  Z_0^t +\nu } }  } \\
&= K q (p/q-1) \left(  \expect{ \frac{ 1 } {1+ \eexp^{-  Z_1^t +\nu } } }    + \eexp^{\nu} \expect{ \frac{1 } {1+ \eexp^{-  Z_0^t +\nu } } }   \right) \\
& = K (p-q),
\end{align*}
where the last equality holds due to \prettyref{eq:symmetry1}. Moreover,
\begin{align}
& \expect{L_u}  \expect{  \frac{ (p/q-1)^2  } { ( 1+ \eexp^{-  Z_1^t +\nu } )^2  }   }+  \expect{M_u}
\expect{  \frac{ (p/q-1)^2  } { ( 1+ \eexp^{-  Z_0^t +\nu } )^2  }   } \nonumber  \\
&= K q (p/q-1)^2 \left(  \expect{  \frac{ 1 } { ( 1+ \eexp^{-  Z_1^t +\nu } )^2  }   } + \eexp^{\nu}  \expect{  \frac{ 1 } { ( 1+ \eexp^{-  Z_0^t +\nu } )^2  }   }  \right) \nonumber \\
&  \overset{(a)}{=}  K q (p/q-1)^2  \expect{  \frac{ 1 } { 1+ \eexp^{-  Z_1^t +\nu }  }  },   \nonumber\\
&  \overset{(b)}{=}  \lambda  \expect{ \frac{ \eexp^{  Z_1^t} }{1 +  \eexp^{Z_1^t -\nu } } }  = \lambda b_t   \label{eq:means}
\end{align}
where (a) holds due to \prettyref{eq:symmetry2}, and (b) holds due to the fact  $\nu = \log \frac{n-K}{n}$.
Also,
\begin{align}
&\bigg|   \expect{L_u}   \expect{  \psi_3 \left(          \frac{ p/q-1 }{  1+ \eexp^{-Z_1^t+\nu } }             \right)    }  + \expect{M_u}
\expect{  \psi_3 \left(  \frac{ p/q-1 }{  1+ \eexp^{-Z_0^t+\nu }  }     }  \right)\bigg|   \nonumber   \\
&\leq   \expect{L_u}   \expect{    \frac{ ( p/q-1)^3 }{  ( 1+ \eexp^{-Z_1^t+\nu } )^3 }     } + \expect{M_u}
\expect{    \frac{ ( p/q-1)^3 }{  ( 1+ \eexp^{-Z_0^t+\nu } )^3 }     } \nonumber  \\
& = K q (p/q-1)^3 \left( \expect{    \frac{ 1 }{  ( 1+ \eexp^{-Z_1^t+\nu } )^3 }     }
+ \eexp^{\nu} \expect{ \frac{ 1 }{  ( 1+ \eexp^{-Z_1^t+\nu } )^3 }     } \right)\nonumber   \\
& \overset{(a)}{=} K q (p/q-1)^3 \expect{    \frac{ 1 }{  ( 1+ \eexp^{-Z_1^t+\nu } )^2 }     } \nonumber  \\
& \le K q (p/q-1)^3 \eexp^{-2\nu } \expect{ \eexp^{2 Z_1^t} } \leq  \frac{\lambda^2 \eexp^{Cb_{t-1}}}{K(p-q)} , \label{eq:smallterm}
\end{align}
where $(a)$ holds due to  \prettyref{eq:symmetry3}; the last inequality holds because, as shown by \prettyref{lmm:exp_BP_bounds}, $\expect{\eexp^{2Z_1^t} }\leq \eexp^{Cb_{t-1}} .$
Assembling the last four displayed equations yields \prettyref{eq:BPZ0_mean}.

Similarly,
\begin{align*}
\expect{Z_1^{t+1}} & = \expect{Z_0^{t+1} } +  K (p-q) \expect{  \log \left(   \frac{   e^{Z_1^t+ \nu}(p/q)+1 }{  e^{Z_1^t-\nu}+1   }   \right)    } \\
& = \expect{Z_0^{t+1} } +  \lambda b_t   + K(p-q)\expect{  \psi_2   \left(  \frac{(p/q)-1}{e^{-Z_1^t + \nu}+1}  \right)   }.
\end{align*}
and, using $|\psi_2(x)|\leq x^2$ and the definition of $\nu,$
$$
\bigg|   K(p-q)\expect{  \psi_2   \left(  \frac{(p/q)-1}{e^{-Z_1^t + \nu}+1}  \right)   } \bigg| \leq  \frac{\lambda^2 \expect{e^{2Z_1^t}}    }{K(p-q)}
\leq   \frac{\lambda^2 \eexp^{Cb_{t-1} }}{K(p-q)}
$$
It follows that \prettyref{eq:BPZ1_mean} holds.

Next, we calculate the variance. For $Y= \sum_{i=1}^{L} X_i$, where $L $ is Poisson distributed
and $\{X_i\}$ are i.i.d.\ with finite second moments, it is well-known that
$\var(Y) = \expect{L} \expect{X_1^2}$.    It follows that
\begin{align*}
& \var \left(Z_0^{t+1} \right) = \expect{L_u} \expect{ \log^2 \left(  \frac{ \eexp^{ Z_1^t -\nu } (p/q)   + 1  }{ \eexp^{ Z_1^t -\nu } +1} \right) }  \\
&~~~~~~~~   + \expect{M_u} \expect{ \log^2 \left(  \frac{ \eexp^{ Z_0^t -\nu } (p/q)   + 1  }{ \eexp^{ Z_0^t -\nu } +1} \right) }.
\end{align*}
Using \prettyref{eq:oneplus} and the fact $|\log^2(1+x)  - x^2| \leq x^3$ (see \prettyref{lmm:logplus}) yields
\begin{align*}
\var \left(Z_0^{t+1} \right) & =
 \expect{L_u}  \expect{  \frac{ (p/q-1)^2  } { ( 1+ \eexp^{-  Z_1^t +\nu } )^2  }   } + \expect{M_u}
\expect{  \frac{ (p/q-1)^2  } { ( 1+ \eexp^{-  Z_0^t +\nu } )^2  }   } \\
& + O \left(  \expect{L_u}   \expect{    \frac{ ( p/q-1)^3 }{  ( 1+ \eexp^{-Z_1^t+\nu } )^3 }     }  + \expect{M_u}
\expect{    \frac{ ( p/q-1)^3 }{  ( 1+ \eexp^{-Z_0^t+\nu } )^3 }     }  \right).
\end{align*}
Applying \prettyref{eq:means} and    \prettyref{eq:smallterm} yields \prettyref{eq:BPZ0_var}.

Similarly,  applying \prettyref{eq:means} and the fact $\log^2(1+x) \leq x^2,$ yields
\begin{align*}
\var \left(Z_1^{t+1} \right) & = \var \left(Z_1^{t+1} \right)   + K (p-q) O \left( \expect{  \frac{ (p/q-1)^2  } { ( 1+ \eexp^{-  Z_1^t +\nu } )^2  }   } \right) \\
& =  \var \left(Z_0^{t+1} \right)    + O \left( \frac{\lambda^2\eexp^{Cb_{t-1}}}{K(p-q)} \right)  \\
&  = \lambda b_t   + O \left(\frac{\lambda^2}{K(p-q)} \right)e^{Cb_{t-1}},
\end{align*}
which together with  \prettyref{eq:BPZ0_var} implies \prettyref{eq:BPZ1_var}.
\end{proof}

\subsection{Asymptotic Gaussian marginals of log likelihood messages}
The following lemma is well suited for proving that the distributions of $Z_0^t$ and $Z_1^t$ are asymptotically Gaussian.
\begin{lemma}[Analog of Berry-Esseen inequality for Poisson sums {\cite[Theorem 3]{korolev2012improvement}}]\label{lmm:Poisson_BE}
Let  $S_{\lambda}=X_1 + \cdots + X_{N_\lambda},$   where
$(X_i: i\geq 1)$  are independent, identically distributed random variables with mean $\mu$, variance $\sigma^2$
and $\expect{|X_i|^3}\leq \rho^3,$ and for some $\lambda > 0,$ $N_{\lambda}$ is a $\Pois(\lambda)$ random variable independent
of  $(X_i: i\geq 1).$   Then
$$
\sup_x \bigg|     \prob{  \frac{S_\lambda - \lambda \mu}{  \sqrt{\lambda(\mu^2 + \sigma^2)}}\leq x} - \Phi(x)   \bigg|  \leq  \frac{C_{BE} \rho^3}{\sqrt{\lambda(\mu^2 + \sigma^2)^3}}
$$
where $C_{BE}=0.3041.$
\end{lemma}

\begin{lemma}\label{lmm:treeGaussian}
Suppose $\lambda  > 0$  is fixed, and the parameters $p/q$ and $\nu$ vary such that
$p/q = O(1),$  $\nu$ is bounded from below (i.e. $K/n$ is bounded away from one)  and
$K(p-q)\to \infty.$  (The latter condition holds if either $\nu \to \infty$ or $p/q\to 1$; see \prettyref{rmk:Kpq}.)
Suppose $t\in \naturals$ is fixed, or more generally, 
$t$ varies with $n$ such that $\frac{\eexp^{C'b_{t-1}}   }{K(p-q)} = o(b_t)$ as $n\to\infty,$
where $C'=\lambda\left(3+2\frac{p}{q}+\left(\frac{p}{q}\right)^2\right).$    Then
\begin{align}
\sup_x  \bigg|  \prob{\frac{Z_0^{t+1} + \frac{\lambda b_t}{2} }{\sqrt{\lambda b_t}}\leq x }   - \Phi(x) \bigg|   \to  0     \label{eq:BPlawZ0}  \\
\sup_x  \bigg|  \prob{\frac{Z_1^{t+1} - \frac{\lambda b_t}{2} }{\sqrt{\lambda b_t}}\leq x }   - \Phi(x) \bigg|   \to  0     \label{eq:BPlawZ1}
\end{align}
\end{lemma}
\begin{remark}
Note that in the case of $\lambda \le 1/\eexp$,
$b_t \le e$ for all $t \ge 0$. As a consequence, 
\prettyref{eq:BPlawZ0}  and \prettyref{eq:BPlawZ1} hold for all $t,$ and, as can be checked from
the proof, the limits hold uniformly in $t.$   Also, in the case $b_t$ is bounded independently of $n$,
\prettyref{eq:BPlawZ1} is a consequence of \prettyref{eq:BPlawZ0} and the fact that $Z_0^{t+1}$ is the log likelihood
ratio.
In the proof below, \prettyref{eq:BPlawZ1} is proved directly.
\end{remark}
\begin{remark}   \label{rmk:Kpq}
The condition $K(p-q)\to \infty$ in \prettyref{lmm:treeGaussian} is essential for the proof; we state some equivalent conditions here.
Equations \eqref{eq:mean_deg11}-\eqref{eq:mean_deg0} express $Kp$, $Kq$, and $(n-K)q$ in terms of the parameters $\lambda, \nu,$ and $p/q.$   Similarly,
\begin{align*}
K(p-q) & =\frac{\lambda \eexp^{\nu}}{p/q - 1}    \\
np & = \frac{\lambda (p/q)  \eexp^{\nu}( \eexp^{\nu}+1)}{(p/q - 1)^2}  \\
\frac{ (n-K)q}{K(p-q)} & = \frac{ \eexp^{\nu}}{p/q - 1}.
\end{align*}
It follows that if $\frac{K^2 (p-q)^2 }{(n-K) q} \equiv \lambda$ for a fixed $\lambda > 0,$   $p/q = O(1)$,
and $\nu$ is bounded below (i.e. $K/n$ is bounded away from one) then the following seven conditions are equivalent:
($K(p-q)\to \infty$),  ($\nu \to \infty$ or $\frac{p}{q}\to 1$),  ($Kp\to\infty$), ($Kq \to \infty$), ($(n-K)q \to\infty$),  ($np\to \infty$), ($K(p-q)=o((n-K)q)$).
\end{remark}
\begin{proof}[Proof of \prettyref{lmm:treeGaussian}]
Throughout the proof it is good to keep in mind that $b_0=\frac{1}{1+ \eexp^{-\nu}}, $  so that $b_0$ is bounded from below by a fixed
positive constant, and, as shown in
\prettyref{lmm:info_monotone},  $b_t$ is nondecreasing in $t.$
For $t\geq 0,$  $Z^{t+1}_0$ can be represented as follows:
\begin{align*}
Z^{t+1}_0  =    - K(p-q)    +   \sum_{i=1}^{N_{nq}}   X_i, 
\end{align*}
where $N_{nq}$ has the $\Pois(nq)$ distribution, the random variables $\{X_i, i\geq 0\}$ are mutually independent and independent of $N_{nq}$, and the
distribution of $X_i$ is a mixture of distributions: $\calL(X_i) = \frac{(n-K)q}{nq}\calL(f(Z_0^t) )+ \frac{Kq}{nq} \calL(f(Z^{t}_1) ) ,$
where $f(z) =\log\left( \frac{ \eexp^{z-\nu}(p/q) + 1}{ \eexp^{z-\nu} + 1}   \right).$

By \prettyref{eq:BPZ0_var} of  \prettyref{lmm:BP_moments} and the formula for the variance of the sum of a Poisson distributed number of iid random variables,
$$
nq  \expect{X_i^2}  =  \var{(Z_0^{1+t}})  = \lambda b_t + O \left(\frac{\lambda^2\eexp^{Cb_{t-1}} }{K(p-q)} \right).
$$
The function $f$, and therefore the $X_i$'s, are nonnegative.    Using the fact $\log^3(1+x)\leq x^3$ for $x\geq 0,$ and
applying \prettyref{eq:oneplus} we find
$f^3(z)   \leq  \left(   \frac{ p/q-1} {1+ \eexp^{- z +\nu}  }\right)^3.$   Applying \prettyref{eq:smallterm} yields
\begin{align}
nq \expect{| X_i|^3} & =    \expect{L_u}   \expect{    \frac{ ( p/q-1)^3 }{  ( 1+ \eexp^{-Z_1^t+\nu } )^3 }     } + \expect{M_u}
\expect{    \frac{ ( p/q-1)^3 }{  ( 1+ \eexp^{-Z_0^t+\nu } )^3 }     } \nonumber  \\
&\leq  \frac{\lambda^2 \eexp^{Cb_{t-1}}}{K(p-q)} .
\end{align}

Therefore, the ratio relevant for application of the Berry-Esseen lemma satisfies:
\begin{align*}
\frac{\expect{|X_i|^3}} {\sqrt{nq \expect{X_i^2}^3 }}  =  \frac{nq\expect{|X_i|^3}} {\sqrt{\left(nq \expect{X_i^2}\right)^3 }}
&  \leq   \frac{ \lambda^2\eexp^{Cb_{t-1}} }{K(p-q) \sqrt{\left(  \lambda b_t   + O \left(\frac{\lambda^2\eexp^{Cb_{t-1}} }{K(p-q)} \right) \right)^3 }  } \to 0.  \\
\end{align*}
The Berry-Esseen lemma, \prettyref{lmm:Poisson_BE}, implies
$$
\sup_x \bigg|       \prob{        \frac{    Z_0^{t+1}  - \expect{Z_0^{t+1}}  } {  \sqrt{\var(Z_0^{t+1})}   }   \leq x                }
- \Phi(x)   \bigg|  \leq  \frac{C_{BE}\expect{|X_i|^3}} {\sqrt{nq \expect{X_i^2}^3 }} .
$$
Applying \prettyref{lmm:BP_moments} completes the proof of \prettyref{eq:BPlawZ0}.

The proof of \prettyref{eq:BPlawZ1} given next is similar.   For $t\geq 0,$  $Z^{t+1}_1$ can be represented as follows:
\begin{align*}
Z^{t+1}_1  =  K(p-q)    +  \frac{1}{\sqrt{(n-K)q} } \sum_{i=1}^{N_{(n-K)q+Kp}}   Y_i 
\end{align*}
where $N_{(n-K)q+Kp}$ has the $\Pois((n-K)q+Kp)$ distribution, the random variables $\{Y_i, i\geq 0\}$ are mutually independent and independent of $N_{(n-K)q+Kp}$, and the
distribution of $Y_i$ is a mixture of distributions: $\calL(Y_i) = \frac{(n-K)q}{(n-K)q+Kp}\calL(f(Z_0^t) )+ \frac{Kp}{(n-K)q+Kp} \calL(f(Z^{t}_1) ) ,$
where $f(z) =\log\left( \frac{ \eexp^{z-\nu}(p/q) + 1}{ \eexp^{z-\nu} + 1}   \right).$

By \prettyref{eq:BPZ1_var} of  \prettyref{lmm:BP_moments} and the formula for the variance of the sum of a Poisson distributed number of iid random variables,
$$
((n-K)q + Kp)  \expect{Y_i^2}  =  \var{(Z_1^{1+t}})  = \lambda b_t + O \left(\frac{\lambda^2}{K(p-q)} \right)e^{Cb_{t-1}}.
$$
We again use $f^3(z)   \leq  \left(   \frac{ p/q-1} {1+ \eexp^{- z +\nu}  }\right)^3.$   Applying \prettyref{eq:smallterm} and \prettyref{lmm:exp_BP_bounds}
yields
\begin{align*}
((n-K)q + Kp)    \expect{| Y_i|^3} & =   nq\expect{|X_i|^3}  + K(p-q)
\expect{    \frac{ ( p/q-1)^3 }{  ( 1+ \eexp^{-Z_1^t+\nu } )^3 }     }  \\
&\leq  \frac{\lambda^2 \eexp^{Cb_{t-1}}}{K(p-q)}  +  \frac{\lambda^3\expect{\eexp^{3Z_1^t}}}{(K(p-q))^2} \\
& \leq \frac{\lambda^2 \eexp^{Cb_{t-1}}}{K(p-q)}  +  \frac{\lambda^3 \eexp^{C'b_{t-1}}}{(K(p-q))^2},
\end{align*}
where $C'=\lambda(3 + 2p/q + (p/q)^2).$

Therefore, the ratio relevant for application of the Berry-Esseen lemma satisfies:
\begin{align*}
\frac{\expect{|Y_i|^3}} {\sqrt{((n-K)q + Kp)  \expect{Y_i^2}^3 }}
&  \leq   \frac{ \lambda^2\eexp^{Cb_{t-1}}    +            \frac{\lambda^3 \eexp^{C'b_{t-1}}}{K(p-q))}                         }{K(p-q) \sqrt{\left(  \lambda b_t   + O \left(\frac{\lambda^2}{K(p-q)} \right)e^{Cb_{t-1}} \right)^3 }  }
\to 0.
\end{align*}
Therefore, the Berry-Esseen lemma,  \prettyref{lmm:Poisson_BE},  along with
\prettyref{lmm:BP_moments}, completes the proof of \prettyref{eq:BPlawZ1}.
\end{proof}

\section{Linear message passing on a random tree}   \label{app:spectral_limit}

\subsection{Linear message passing on a random tree--exponential moments}\label{sec:mgf}
To analyze the message passing algorithms given in \prettyref{eq:spectral_mp1} and \prettyref{eq:spectral_mp2}, we first study
an analogous message passing algorithm on the tree model introduced in \prettyref{sec:BP_Bernoulli_tree}:
\begin{align}
\xi_{i\to \pi(i)}^{t+1} &=- \frac{q ((n-K)A_t  + K B_t) }{\sqrt{m}}  +     \frac{1}{\sqrt{m} } \sum_{\ell \in \partial i } \xi^t_{ \ell \to i },   \label{eq:spectral_tree_mp1} \\
\xi_{u}^{t+1} & =   - \frac{q ((n-K)A_t  + K B_t) }{\sqrt{m}}  +     \frac{1}{\sqrt{m} } \sum_{i \in \partial u } \xi^t_{ i \to u },   \label{eq:spectral_tree_mp2}
\end{align}
with initial values $\xi^0_{ \ell \to \pi(\ell) }=1$ for all $\ell \neq u$,  where $\pi(\ell)$ denotes the parent of $\ell$, 
and $m = (n-K) q$. 
Let $Z^t_0$ denote a random variable that has the same distribution as $\xi_u^t$ given $\tau_u=0,$
and let $Z^t_1$ denote a random variable that has the same distribution as $\xi_u^t$ given $\tau_u=1.$
Equivalently, $Z^t_b$ for $b \in \{0,1\}$ has the distribution of $\xi^t_{\ell\to\pi(\ell)}$ for any vertex $\ell \neq u$,  given
$\tau_{\ell}=b.$
Let $A_t=\expect{Z_0^t}$ and $B_t=\expect{Z_1^t}.$ Then $A_0=B_0=1$. 
Given $\tau_u=0,$  the mean of the sum in \prettyref{eq:spectral_tree_mp1} is subtracted out, so $A_t = \expect{Z^t_0}=0$ for
all $t\geq 1.$    Compared to the case $\tau_u = 0$,  if $\tau_u = 1$, then on average there are $K(p-q)$ additional
children of node $u$ with labels equal to 1, so that $B_{t+1}= \sqrt{\lambda} B_t,$ which gives $B_t=\lambda^{t/2}$ for
$t\geq 0.$

We consider sequences of parameter triplets $(\lambda, p/q, K/n)$ indexed by $n.$
Let $\psi_i^t(\eta) = \expect{e^{\eta Z_i^t } }$ for $i=0,1$ and $t\ge 1$.
Expressions are given for these functions when $t=1$ in \prettyref{eq:psi01_def}
and \prettyref{eq:psi11_def} below.   Following the same method used in
\prettyref{sec:BP_Bernoulli_tree} for the belief propagation algorithm,
we find the following recursions for $t\geq 1:$
\begin{align}
\psi_0^{t+1}(\eta) = & ~ \exp\sth{m \pth{\psi_0^t\pth{\frac{\eta}{\sqrt{m}}} - 1}  +  Kq \pth{\psi_1^t\pth{\frac{\eta}{\sqrt{m}}} - 1  - \frac{\eta}{\sqrt{m}} \lambda^{t/2}  }}, \label{eq:recursive_psi1} \\
\psi_1^{t+1}(\eta) = & ~ \psi_0^{t+1}(\eta) \exp\sth{ \sqrt{\lambda m} \pth{\psi_1^t\pth{\frac{\eta}{\sqrt{m}}} - 1   } }. \label{eq:recursive_psi2} 
\end{align}

\begin{lemma}\label{lmm:mgf}
Assume that as $n\to \infty,$   $\lambda$ is fixed, $K=o(n),$ and $p/q=O(1).$
(Consequently, $m\to \infty$; see \prettyref{rmk:Kpq}.)   Let $\gamma$ be a constant such that $\gamma > 1$ and  $\gamma \geq \lambda.$
Let  $T = 2 \alpha \frac{\log \frac{n-K}{K}}{\log \gamma},$ where $\alpha=1/4$ (in fact any $\alpha < 1$ works).
Let $c = \frac{1}{4} \log \gamma$ (in fact any $c\in (0, \log \sqrt{\gamma})$ works).
For sufficiently large $n$, $t\in [T],$ and $\eta$ such that $\gamma^{(t-1)/2} (\frac{\eta^2}{m} + \frac{\eta}{\sqrt{m}}) \leq c,$
\begin{align}
\psi_0^t(\eta) \leq & ~ \exp( \gamma^{t/2} \eta^2),	\label{eq:psi0} \\
\psi_1^t(\eta) \leq & ~ \exp( \lambda^{t/2}  \eta + \gamma^{t/2} \eta^2 ).	\label{eq:psi1}
\end{align}	
\end{lemma}

\begin{proof}[Proof of \prettyref{lmm:mgf}]
Recall that $m=(n-K)q$ and $K(p-q) = \sqrt{\lambda m}$. Since $K=o(n)$, it follows that $(nq)/m \to 1$. Also, because $\lambda$
is fixed, we have that $\lambda/m \to 0$. Hence, the choice of $c$ ensures that for $n$ sufficiently large,
\begin{align}
\pth{ \frac{nq}{m}    +  \sqrt{\frac{\lambda}{m} }}  \frac{e^{c}}{2} \leq & ~ \sqrt{\gamma}	\label{eq:c1}.
\end{align}
By \prettyref{eq:spectral_tree_mp1}, $\xi_{i \to \pi(i) }^{1} = \frac{-nq + |\partial i| }{\sqrt{m}}$. Hence, for $t=1$ and $\eta \in (-\infty, \sqrt{m}c]$
\begin{align}
\psi_0^1(\eta)
= & ~ \exp(nq(e^{\eta/\sqrt{m}} - 1 - \eta/\sqrt{m} ))	\label{eq:psi01_def}  \\
\leq & ~ \exp\pth{ \frac{nq}{2m} e^{c} \eta^2} \overset{\prettyref{eq:c1}}{\leq} \exp(\sqrt{\gamma} \eta^2) \nonumber,
\end{align}
where we used the fact that $e^x \leq 1 + x + \frac{e^c}{2} x^2$ for all $x \in (-\infty, c].$  Similarly,
\begin{align}
\psi_1^1(\eta)
= & ~ \psi_0^1(\eta) \exp(K(p-q)(e^{\eta/\sqrt{m}} - 1))	\label{eq:psi11_def} \\
\leq & ~ \exp\pth{ \frac{nq}{2m} e^{c} \eta^2}  \exp\pth{ \sqrt{\lambda m} \pth{\frac{\eta}{\sqrt{m}} + \frac{e^c}{2} \frac{\eta^2}{m} }}  \nonumber \\
\leq & ~ \exp\pth{  \sqrt{\lambda} \eta + \pth{ \frac{nq}{m}    +  \sqrt{\frac{\lambda}{m} }} \frac{e^{c}}{2} \eta^2} \overset{\prettyref{eq:c1}}{\leq} \exp(\sqrt{\lambda} \eta + \sqrt{\gamma} \eta^2 ). \nonumber
\end{align}
Thus, \prettyref{eq:psi0} and \prettyref{eq:psi1} hold for $t=1$ and $\eta$ as described in the lemma.

Observe that
\begin{equation}
	\gamma^{T/2} \frac{1}{\sqrt{m}} = o(1),  
	\label{eq:lambdaT}
\end{equation}
because $\gamma^{T/2} \frac{1}{\sqrt{m}} = (\frac{n-K}{K})^{\alpha} \frac{1}{\sqrt{m}} = \lambda^{-\alpha/2} (\frac{p}{q}-1)^\alpha  m^{-\frac{1-\alpha}{2}} = o(1)$.
In addition, the choice of $c$ guarantees that, for $n$ sufficiently large,
\begin{equation}
 e^c + \pth{\frac{Kq}{m}  + \sqrt{\frac{\lambda}{m}} }   \pth{1 +  \frac{e^c}{2} \pth{   3 c + \gamma^{T/2}} }
	 \leq \sqrt{\gamma},
	\label{eq:c2}
\end{equation}
because $\frac{Kq}{m} = o(1)$, $m\diverge$, $\frac{Kq}{m} \gamma^{T/2} = (\frac{K}{n-K})^{1-\alpha} = o(1),$ and \prettyref{eq:lambdaT} holds.
Assume for the sake of proof by induction that,  for some $t$ with $1\leq t  < T,$
\prettyref{eq:psi0} and  \prettyref{eq:psi1} hold for all $\eta \in \Gamma_t \triangleq \{\eta: \gamma^{(t-1)/2} (\frac{\eta^2}{m} + \frac{\eta}{\sqrt{m}}) \leq c\}$ . Now fix $\eta \in \Gamma_{t+1}$. Since 
$\Gamma_t$ is an interval containing zero for each $t$ and $\Gamma_{t+1}\subset\Gamma_t,$
 it is clear that $\frac{\eta}{\sqrt{m}} \in \Gamma_t$ for $m\geq 1.$  By \prettyref{eq:recursive_psi1}, we have
\begin{align*}
\log \psi_0^{t+1}(\eta)
= & ~ m \pth{\psi_0^t\pth{\frac{\eta}{\sqrt{m}}} - 1}  +  Kq \pth{\psi_1^t\pth{\frac{\eta}{\sqrt{m}}} - 1  - \frac{\eta}{\sqrt{m} } \lambda^{t/2} } \\
\leq & ~ m \pth{e^{\frac{\gamma^{t/2} \eta^2}{m}} - 1}  +   Kq \pth{e^{  \frac{\gamma^{t/2} \eta^2}{m} + \lambda^{t/2} \frac{\eta}{\sqrt{m}} } - 1  - \frac{\eta}{\sqrt{m} } \lambda^{t/2} } \\
\leq & ~ e^c \gamma^{t/2} \eta^2 +  Kq \pth{   \frac{\gamma^{t/2} \eta^2}{m} + \frac{e^c}{2}\pth{  \frac{\gamma^{t/2} \eta^2}{m} + \lambda^{t/2} \frac{\eta}{\sqrt{m}} }^2 }   \\
\leq & ~ \gamma^{t/2} \eta^2 \pth{e^c + \frac{Kq}{m}} +  \frac{Kq}{2m} e^c \pth{   3 c \gamma^{t/2}  + \gamma^{t}} \eta^2  \\
\overset{\prettyref{eq:c2}}{\leq} & ~ \gamma^{(t+1)/2} \eta^2,
\end{align*}
where the first inequality holds due to the induction hypothesis; the second inequality holds due to  $e^x \leq 1 + e^c x $ for all $x \in [0, c]$
and $e^x \leq 1 + x + \frac{e^c}{2} x^2$ for all $x \in (-\infty, c]$; the third inequality holds due to the fact that $\eta \in \Gamma_{t+1}$ and $\lambda \le \gamma$.
 Similarly,
\begin{align*}
&\sqrt{\lambda m} \pth{\psi_1^t\pth{\frac{\eta}{\sqrt{m}}} - 1   } \\
&\leq  ~ \sqrt{\lambda m}  \pth{   \frac{\gamma^{t/2} \eta^2}{m} + \frac{e^c}{2}\pth{  \frac{\gamma^{t/2} \eta^2}{m} + \lambda^{t/2} \frac{\eta}{\sqrt{m}} }^2 + \frac{\eta}{\sqrt{m} } \lambda^{t/2}} \\
&=  ~ \sqrt{\frac{\lambda}{m}}  \pth{	\gamma^{t/2}  +  \frac{e^c}{2} \pth{   3 c \gamma^{t/2}  + \gamma^{t}}    } \eta^2 +  \lambda^{(t+1)/2} \eta
\end{align*}
and hence by  \prettyref{eq:recursive_psi2},
\begin{align*}
& \log \psi_1^{t+1}(\eta) 
=   \log \psi_0^{t+1}(\eta) + \sqrt{\lambda m} \pth{\psi_1^t\pth{\frac{\eta}{\sqrt{m}}} - 1   }  \\
& \leq \gamma^{t/2} \eta^2 \pth{ e^c + \frac{Kq}{m}  + \sqrt{\frac{\lambda}{m}} } +  \pth{ \frac{Kq}{m}  + \sqrt{\frac{\lambda}{m}} } \frac{e^c}{2} \pth{   3 c \gamma^{t/2}  + \gamma^{t}} \eta^2 + \lambda^{\frac{t+1}{2} } \eta \\
& \overset{\prettyref{eq:c2}}{\leq}  \lambda^{(t+1)/2} \eta + \gamma^{(t+1)/2} \eta^2 .
\end{align*}
\end{proof}

\begin{corollary}  \label{cor:mgb_recovery}
Assume that as $n\to \infty,$   $\lambda$ is fixed with $\lambda > 1$, $K=o(n),$ and $p/q=O(1).$
Let  $T = 2 \alpha \frac{\log \frac{n-K}{K}}{\log \lambda},$ where $\alpha=1/4.$
If $\tau = \frac{1}{2} \lambda^{T/2}$,
then $\prob{Z_0^T  \geq \tau} =  o(\frac{K}{n-K})$ and $\prob{Z_1^T  \leq \tau}= o(\frac{K}{n-K}).$
\end{corollary}

\begin{proof}
Since $\lambda>1$  we can let $\gamma=\lambda$ in  \prettyref{lmm:mgf} so that $T$ here is the same as $T$ in \prettyref{lmm:mgf}.
Equation \prettyref{eq:lambdaT} implies that the interval of $\eta$ values satisfying the condition of \prettyref{lmm:mgf}  for $t=T$ converges
to all of $\reals.$  By \prettyref{lmm:mgf} and the Chernoff bound for threshold at $\tau = \frac{1}{2} \lambda^{T/2},$  for any $\eta> 0$, if $n$ is sufficiently
large
\begin{align}
\prob{Z_0^T  \geq \tau} & ~ \leq \psi_0^T(\eta) \exp(-\eta\tau)\leq \exp( \lambda^{T/2} (\eta^2-\eta/2)) \overset{\eta=1/4}{=} \exp( - \lambda^{T/2}/16 ).
\end{align}
Similarly, for any $\eta < 0$ and $n$ sufficiently large,
\begin{align}
\prob{Z_1^T  \leq \tau} & ~ \leq \psi_1^T(\eta) \exp(-\eta\tau)\leq \exp( \lambda^{T/2} (\eta^2+\eta/2)) \overset{\eta=-1/4}{=} \exp( - \lambda^{T/2}/16 ).
\end{align}
By the choice of $T$, we have $\lambda^{T/2} = (\frac{n-K}{K})^{\alpha}$ and hence $\exp( - \lambda^{T/2}/16 ) = o(\frac{K}{n-K}).$
\end{proof}

\subsection{Gaussian limits of messages}  \label{sec:linear_Gaussian_limits}

In this section we apply the bounds derived in \prettyref{sec:mgf}
and  a version of the Berry-Esseen central limit theorem for compound Poisson sums
to show the messages are asymptotically Gaussian.
As in \prettyref{sec:mgf}, the result allows the number of iterations
to grow slowly with $n.$

Let $\alpha_t=\var(Z_0^t)$ and  $\beta_t=\var( Z_1^t).$
Using the usual fact $\var(\sum_{i=1}^X Y_i)= \expect{X}\var(Y)+ \var(X) \expect{Y}^2$ for iid $Y$'s,  we find
\begin{align}
\alpha_{t+1}& =   \alpha_t  + A_t^2  + \frac{Kq}{m}   \beta_t + \frac{Kq}{m} B_t^2  \label{eq:alpha1}  \\
\beta_{t+1}& =    \alpha_t  +  A_t^2  +   \frac{Kp}{m}   \beta_t +  \frac{Kp}{m}  B_t^2  \label{eq:beta1}
\end{align}
with the initial conditions $\alpha_0=\beta_0=0.$
Comparing the recursions (without using induction) shows that $\alpha_t \leq \beta_t \leq \frac{p}{q} \alpha_t$  for $t\geq 0.$
Note that $\alpha_1=\frac{n}{n-K}\geq 1$, and $\alpha_t$ is
nondecreasing in $t$.  Thus $1\leq \alpha_t \leq \beta_t$ for all $t.$
Therefore, if $\lambda < 1$, the signal to noise
ratio $\frac{(B_t -A_t)^2}{\alpha_t} \leq \lambda^t \to 0$ as $t\rightarrow \infty.$
Also, under the assumption $K=o(n)$ and $p/q=O(1),$ the coefficients in the recursions \prettyref{eq:alpha1} and \prettyref{eq:beta1}
satisfy $ \frac{Kq}{m} \to 0$ and  $\frac{Kp}{m} \to 0$ as $n \to \infty.$
Thus,  $\alpha_t \to 1 $ and $\beta_t \to 1$ for $t$ fixed as $n\to \infty.$

The following lemma proves that the distributions of $Z_0^t$ and $Z_1^t$ are asymptotically Gaussian.
\begin{lemma}\label{lmm:treeGaussianspectral}
Suppose that  as $n\to\infty,$  $\lambda$ is fixed with $\lambda>0$,   
$K=o(n),$  $p/q=O(1)$, and $t$ varies with $n$ such that
$t \in \naturals$  and the following holds:   If $\lambda > 1$ then $\lambda^{t/2}\leq \left( \frac{n-K}{K}\right)^{\alpha},$  where $\alpha=1/4$ (any
$\alpha \in (0,1/3)$ works), and if $\lambda \leq 1$:  $t=O(\log\left(  \frac{n-K}{K}\right) ).$   Then as $n\to \infty,$
\begin{align}
\sup_x  \bigg|   \prob{     \frac{ Z^t_0}{ \sqrt{ \alpha_t } }                           \leq x }  - \Phi(x)  \bigg|  \to 0    \label{eq:lawZ0}   \\
\sup_x  \bigg|   \prob{     \frac{ Z^t_1-\lambda^{t/2}}{ \sqrt{ \beta_t } }    \leq  x }  - \Phi(x)  \bigg|  \to 0 .   \label{eq:lawZ1}
\end{align}
\end{lemma}

\begin{proof}  Select a constant $\gamma > 1$ as follows.   If  $\lambda > 1, $  let $\gamma=\lambda.$
If $\lambda \leq 1$,  select $\gamma > 1$ so that  $\gamma^{t/2}\leq \left( \frac{n-K}{K}\right)^{\alpha}$ for all $n$ sufficiently large,
which is possible by the assumptions.  Then no matter what the value of $\lambda$ is,   $\gamma^{t/2}\leq \left( \frac{n-K}{K}\right)^{\alpha}.$
Let $T$ be defined as in \prettyref{lmm:mgf}.    Since $\gamma^{t/2}\leq \left( \frac{n-K}{K}\right)^{\alpha}$ it follows that $t\leq T.$

For $t\geq 0,$  $Z^{t+1}_0$ can be represented as follows:
\begin{align*}
Z^{t+1}_0  =    -  \frac{Kq\lambda^{t/2}  + (n-K) q \indc{t=0} } {\sqrt{m} }   +   \frac{1}{\sqrt{m} } \sum_{i=1}^{N_{nq}}   X_i
\end{align*}
where $N_{nq}$ has the $\Pois(nq)$ distribution, the random variables $X_i, i\geq 0$ are mutually independent and independent of $N_{nq}$, and the
distribution of $X_i$ is a mixture of distributions: $\calL(X_i) = \frac{(n-K)}{n}\calL(Z^{t}_0) + \frac{K}{n} \calL(Z^{t}_1) .$

Note that  $\expect{ |X_1|^3 } \leq  \max\{ \expect{|Z_0^t|^3},\expect{|Z_1^t|^3} \} \triangleq \rho^3$. 
By \prettyref{lmm:Poisson_BE}, 
\begin{align*}
& \sup_x \bigg|     \prob{  \frac{\sqrt{m} Z_0^{t+1} +Kq\lambda^{t /2}  +(n-K) q \indc{t=0} -  nq\expect{X_1}}{  \sqrt{nq\expect{X_1^2}}}\leq x} - \Phi(x)   \bigg|  \\ & \leq  \frac{C \rho^3}{\sqrt{nq  \expect{X_1^2}^3}} .
\end{align*}
Using the fact $\Expect[X_1^2] \geq 1$,  $\expect{X_1}=\frac{K}{n}\lambda^{t/2} + \frac{n-K}{n} \indc{t=0},$ and $\frac{n}{n-K}\expect{X_1^2}  = \alpha_{t+1}$, we obtain
$$
\sup_x \bigg|     \prob{  \frac{   Z_0^{t+1}  }{  \sqrt{ \alpha_{t+1}  }  }\leq x} - \Phi(x)   \bigg|  \leq  \frac{C \rho^3}{\sqrt{nq}}.
$$
Equation \prettyref{eq:lambdaT} implies that the interval of $\eta$ values satisfying the condition of \prettyref{lmm:mgf}  for $t \leq T$ converges
to all of $\reals.$

In view of \prettyref{lmm:mgf} and the fact $\gamma \geq \max\{\lambda,1\}$, we have that for $n$ sufficiently large,
$\psi_0^t( \pm \gamma^{-t/2}) \leq  ~ 1$ and 
$\psi_1^t( \pm \gamma^{-t/2} ) \leq  ~ e^2.$
Applying  $e^{x}+e^{-x} \ge |x|^3/6$ with $x=Z_0^t/\gamma^{t/2}$ or  $x=Z_1^t/\gamma^{t/2}$ yields:
\begin{align*}
\expect{ |Z_0^t|^3 } &\le 6 \gamma^{3t/2} \left( \psi_0^t (\gamma^{-t/2}) + \psi_0^t (-\gamma^{-t/2}) \right) \le 12 \gamma^{3t/2}, \\
\expect{ | Z_1^t| ^3 } &\le 6 \gamma^{3t/2} \left(  \psi_1^t (\gamma^{-t/2}) + \psi_1^t (-\gamma^{-t/2}) \right).
\le 12 \eexp^2 \gamma^{3t/2}
\end{align*}   
Since $\lambda \leq \left(  \frac{K}{n-K} \right)^2 nq   \left(\frac{p}{q} \right)^2$ it follows that $\sqrt{nq} = \Omega(n/K).$
Hence, $\frac{\rho^3}{\sqrt{nq}}  = O(  \left( \frac{n-K}{K}\right)^{3 \alpha} \frac{K}{n}) =
O\left(  \left( \frac{K}{n} \right)^{1-3\alpha} \right)\to 0$ and \prettyref{eq:lawZ0} follows.

The proof of \prettyref{eq:lawZ1} given next is similar.   For $t\geq 0,$  $Z^{t+1}_1$ can be represented as follows:
\begin{align*}
Z^{t+1}_1  =  -  \frac{Kq\lambda^{t/2} +  (n-K) q \indc{t=0}  }{\sqrt{m} }   +  \frac{1}{\sqrt{m} } \sum_{i=1}^{N_{(n-K)q+Kp}}   Y_i
\end{align*}
where $N_{(n-K)q+Kp}$ has the $\Pois((n-K)q+Kp)$ distribution, the random variables $Y_i, i\geq 0$ are mutually independent and independent of $N_{(n-K)q+Kp}$, and the
distribution of $Y_i$ is a mixture of distributions: 
$$
\calL(Y_i) = \frac{m}{m+Kp}\calL(Z^{t}_0) + \frac{Kp}{m+Kp} \calL(Z^{t}_1) .
$$
Note that 
$\expect{ |Y_1|^3 } \leq  \max\{ \expect{|Z_0^t|^3},\expect{|Z_1^t|^3} \} = \rho^3$.
\prettyref{lmm:Poisson_BE} therefore implies
\begin{align*}
&\sup_x \bigg|     \prob{  \frac{\sqrt{m}Z_1^{t+1}   +Kq\lambda^{t/2} + m \indc{t=0}  - (m+Kp)\expect{Y_1}}{  \sqrt{(m+Kp)\expect{Y_1^2}}}\leq x} - \Phi(x)   \bigg|  \\
&~~~~~~~~~~~~~~~~   \leq  \frac{C \rho^3}{\sqrt{(m+Kp)  \expect{Y_1^2}^3}}
\end{align*}
Using the facts $\Expect[Y_1^2] \geq 1$, $p > q$,  $\expect{Y_1}=\frac{Kp}{m+Kp}\lambda^{t/2} + \frac{m}{m+Kp}  \indc{t=0} ,$  
and $\frac{(m+Kp)}{m}\expect{Y_1^2}  =  \beta_{t+1},$ we obtain
$$
\sup_x \bigg|     \prob{  \frac{   Z_1^{t+1} - \lambda^{(t+1)/2}}{  \sqrt{ \beta_{t+1} } }  \leq x } - \Phi(x)   \bigg|   \leq  \frac{C \rho^3}{\sqrt{nq}}
$$
and the desired  \prettyref{eq:lawZ1} follows.
\end{proof}

\subsection{Proofs for linear message passing}

\begin{proof}[Proof of \prettyref{thm:SP_bernoulli_spectral} ] 
The proof consists of combining \prettyref{cor:mgb_recovery}   and the coupling lemma. Let $T= \frac{1}{2} \log \frac{n-K}{k} / \log \lambda$. 
By the assumption that $np^{\log (n/K)} =n^{o(1)}$ and $\nu=n^{o(1)}$, 
it follows that Therefore, $(2+np)^T = n^{o(1)};$  the coupling lemma can be applied.
The performance bound of \prettyref{cor:mgb_recovery} is for a
hard threshold rule for detecting the label of the root  node.   The same rule could be implemented
at each vertex of the graph $G$ which has a locally tree like neighborhood of radius $T$  by using
the estimator $\hat C_o = \{ i : \theta_{i}^T \geq \lambda^{T/2}/2 \}.$   We first bound the performance for $\hat C_o $
and then do the same for $\hat C$ produced by Algorithm \ref{alg:SP_commun}.

The average probability of misclassification of any given vertex $u$ in $G$  by $\hat C_o$  (for
prior distribution  $(\frac{K}{n},\frac{n-K}{n})$)  is less than or equal to the sum of two
terms.   The first term is less than or equal to $n^{-1/2 + o(1)}$ (due to coupling error) by \prettyref{lmm:treecoupling}.
The second term is  $o(\frac{K}{n-K})$ (due to error of classification of the root vertex of the Poisson tree
graph of depth $T$)  by  \prettyref{cor:mgb_recovery}.     Multiplying the average error probability by $n$ bounds the
expected total number of misclassification errors,  $ \expect{  |C^* \triangle \hat{C}_o|  }.$
By the assumption that $K=n^{1+o(1)},$  so $n^{-1/2 + o(1)}\frac{n}{K}=n^{-1/2 + o(1)} = o(1),$
and of course $o(\frac{K}{n-K})\frac{n}{K}=o(1).$   It follows that $ \frac{   \expect{  |C^* \triangle \hat{C}_o|  }  }{K}  \to 0.$
The set $\hat C_o$ is defined by a threshold condition whereas $\hat C$ similarly corresponds to using a
data dependent threshold and tie breaking rule to arrive at $|\hat C| \equiv K.$
By the same method used in the proof of \prettyref{thm:BP_Bernoulli},
the conclusion for $\hat C$ follows from what was proved for $\hat C_o.$
\end{proof}

The proof of the converse result for linear message passing are quite similar to the proofs of converse results for belief propagation,
and thus are omitted. 
The main differences are that the means here are $0$ and $\lambda^{t/2}$
instead of $\pm b_t/2,$  and the variances here are unequal: $\alpha_t$ and $\beta_t.$   However, since
$\alpha_t \leq \beta_t \leq  \frac{ \alpha_t p}{q}$ and we assume $p/q=O(1),$  the same arguments go through.
Finally, the messages in the linear message passing algorithm do not correspond to log likelihood messages,
and the number of iterations needs to satisfy the extra constraint: $t=O \left( \log \frac{n-K}{K} \right).$

\end{appendix}

\end{document}